\pdfoutput=1
\documentclass{article}
\usepackage{arxiv}
\usepackage[hyphens]{url}  % DO NOT CHANGE THIS
\usepackage{graphicx} % DO NOT CHANGE THIS
\usepackage{algorithm}
\usepackage{algorithmic}
\usepackage{newfloat}
\usepackage{listings}

\usepackage{multirow}
\usepackage[T1]{fontenc}    % use 8-bit T1 fonts
\usepackage{hyperref}       % hyperlinks
\usepackage{booktabs}       % professional-quality tables
\usepackage{amsfonts}       % blackboard math symbols
\usepackage{nicefrac}       % compact symbols for 1/2, etc.
\usepackage{microtype}      % microtypography
\usepackage{xcolor}         % colors
\usepackage{amsmath,amsthm,amssymb}
\usepackage{bm}
\usepackage{pifont}
\usepackage[inline]{enumitem}
\usepackage{natbib}
\usepackage{array}
\newcolumntype{H}{>{\setbox0=\hbox\bgroup}c<{\egroup}@{}}

\newcommand{\rsqrt}[1]{#1^{-\frac{1}{2}}}
\newcommand{\RR}{I\!\!R} %real numbers
\newcommand{\CC}{I\!\!\!\!C} %complex numbers
\newcommand{\HH}{I\!\!\!\!H} %quaternion

\DeclareMathOperator{\Diag}{Diag}
\DeclareMathOperator{\diag}{diag}

\DeclareMathOperator{\sgn}{sgn}
\newcommand{\red}[1]{\textcolor{red}{#1}}
\newcommand{\green}[1]{\textcolor{green}{#1}}

\newtheorem{theorem}{Theorem}
\newtheorem{corollary}{Corollary}

\renewcommand{\checkmark}{\green{\ding{51}}}
\newcommand{\xmark}{\red{\ding{53}}}
\newcommand{\ii}{\mathbf{i}}
\newcommand{\signum}{SigMaNet}

\newcommand{\quater}{QuaterGCN}
\newcommand{\laplaciano}{Quaternionic Laplacian}
\usepackage{scalerel,amssymb}

\newcommand{\jj}{\mathbf{j}}
\newcommand{\kk}{\mathbf{k}}
\usepackage{tikz}
\usetikzlibrary{graphs}
\usepackage[greek,british]{babel}
\usepackage{teubner}
\newcommand{\symbolapl}{L^{\text{\coppa}}}
\newcommand{\ku}{\text{\coppa}}

\newcommand{\Z}{\mathbb{Z}}

\title{Graph Learning in 4D: a Quaternion-valued Laplacian to Enhance Spectral GCNs}

\author{%
  Stefano Fiorini\\
   Italian Institute of Technology\\
   Genoa, Italy  \\
   \texttt{stefano.fiorini@iit.it} \\
  \And
  Stefano Coniglio \\
  University of Bergamo \\
  Bergamo, Italy \\
  \texttt{stefano.coniglio@unibg.it} \\
  \AND
  Michele Ciavotta \\
  University of Milano-Bicocca \\
  Milan, Italy \\
  \texttt{michele.ciavotta@unimib.it} \\
  \And
  Enza Messina\\
  University of Milano-Bicocca \\
  Milan, Italy \\
  \texttt{enza.messina@unimib.it} \\
}

\begin{document}
\maketitle

\begin{abstract}

We introduce {\em \quater{}}, a spectral Graph Convolutional Network (GCN) with quaternion-valued weights at whose core lies the {\em \laplaciano{}}, a quaternion-valued Laplacian matrix by whose proposal we generalize two widely-used Laplacian matrices: the {\em classical Laplacian} (defined for undirected graphs) and the complex-valued \textit{Sign-Magnetic Laplacian} (proposed to handle digraphs with weights of arbitrary sign).
In addition to its generality, our \laplaciano{} is the only Laplacian to completely preserve the topology of a digraph,
as it can handle graphs and digraphs containing antiparallel pairs of edges (digons) of different weights without reducing them to a single (directed or undirected) edge as done with other Laplacians.
Experimental results show the superior performance of \quater{} compared to other state-of-the-art GCNs, particularly in scenarios where the information the digons carry is crucial to successfully address the task at hand.

\end{abstract}
\section{Introduction}

Deep Learning (DL) has recently achieved a striking success, contributing to the advancement of several research areas such as natural language processing \citep{NIPS2017_3f5ee243}, business intelligence~\citep{khan2020machine}, and cybersecurity~\citep{dixit2021deep}, only to name a few.
While many of the most popular DL architectures are designed to process data arranged in grid-like structures (such as RGB pixels in 2D images and video streams), many real-world phenomena are ruled by more general relationships that are better represented by a graph~\citep{bronstein2017geometric}.
For example, e-commerce relies on graphs to model the user-product interaction to make recommendations, drug discovery models molecule bioactivity via interaction graphs, and information discovery uses multi-relational knowledge graphs to extract information between the entities~\citep{knowledgegraph2022}.

Graph Convolutional Networks (GCNs)
%capture and successfully exploit the topology of the graph,
model the interaction between the entities of a complex system via a graph-based {\em convolution operator}. They are the primary tool for machine learning tasks on data that enjoy a graph-like structure.
In spatial GCNs, the convolution operator is defined as a localized aggregation operator~\citep{xu2018powerful}. There are various implementations and extensions of this operator. E.g., in attention-based GCNs the convolutional operator is used to dynamically learn attention scores on the graph edges~\citep{veličković2018graph}, while, in recurrence-based GCNs, the operator leverages gated recurrent units to capture edge-specific information and temporal dependencies during message passing~\citep{li2016gated}.

While spatial GCNs rely on different heuristics for the definition of their convolution operator, {\em spectral} GCNs are solidly grounded in signal and algebraic graph theory. They rely on a convolution operator based on a Fourier transform applied to the eigenspace of the Laplacian matrix of the graph~\citep{kipf2016semi,he2022pytorch,wu2022graph} to capture the global structure of the graph.
%Indeed, besides being mathematically principled, s
Spectral GCNs have been shown to achieve superior performance to their spatial-based counterpart in a number of papers, see, e.g., \cite{zhang2021magnet, fiorini2022sigmanet}.
Our focus in the paper will be on extending the applicability of spectral GCNs. Spatial GCNs will be considered only for comparison purposes within computational experiments.

Despite the growing interest in academia and industry alike for neural-network methods suitable for progressively more general classes of graphs, the literature on spectral GCNs has only recently begun to include extensions beyond the basic case of unweighted, undirected graphs. As a result, the adoption of spectral CGNs in learning tasks involving graphs such as those featuring directed edges, digons, and edges with negative weights~\citep{he2022pytorch,wu2022graph} is still limited.
To (partially) address such limitations, alternative notions of the Laplacian matrix have been put forward, among which the \textit{Magnetic Laplacian}~\citep{lieb1993fluxes}, which was originally proposed in physics and first used within a spectral GCN in~\citep{zhang2021magnet}, albeit with the limitation of only handling digraphs non-negative edge weights.
Spectral-based GCNs suitable for digraphs with weights of unrestricted sign have been proposed only recently. Three such networks are SigMaNet~\citep{fiorini2022sigmanet} (based on the therein proposed {\em Sign-Magnetic Laplacian}), MSG-NN~\citep{he2022msgnn} (based on an extension of the {\em Magnetic Laplacian}), and the network proposed in~\citep{ko2022graph}.

%To the best of our knowledge, no spectral-based GCNs exist that can handle graphs with digons of arbitrary weights and signs without reducing them to single edges (either directed or undirected), thereby partially destroying the graph's topology.
%
Graphs with {\em digons} (pairs of antiparallel directed edges) occur in many important applications including, e.g., traffic and networking and, so far, have eluded every spectral-based GCN proposed in the literature.
Indeed, to the best of our knowledge none of the known spectral-based GCNs can handle graphs with digons with arbitrary weights and signs without collapsing them to single edges (either directed or undirected), which often results in destroying any topological information such digons may carry.

%STRESSARE: SPECTRAL METHODS ARE MATHEMATICALLY PRINCIPLED (DIFFERENTLY FROM SPACIAL METHODS IN WHICH LITERALLY EVERYTHING IS ALLOWED) AND HAVE BEEN SHOWN (CITARE PAPERI PRECEDENTI + NOSTRI RISULTATI DI QUESTO PAPERO) CHE PERFORMANO PEGGIO DI QUELLI SPETTRALI. LA ”TEORIA” SOTA PER GLI SPETTRALI NON GESTISCE I DIGONS. PER ESTENDERLA A QUESTO CASO INTRODUCIAMO UN LAPLACIANO NUOVO QUATERNION-VALUED CHE IN QUALCHE FORMA GENERALIZZA ANCHE I PRECEDENTI TOGLIERE ENFASI DAL PROPORRE UN OGGETTO NUOVO E GENERALE E METTERLA SUL PROPORRE UN OGGETTO CHE `E GENERALE ”PER CASO” MA CHE `E DISEGNATO PER GESTIRE I DIGONS

%In this work, we make a step forward within the theory of spectral GCNs, and propose the \laplaciano{} $\symbolapl$, the first graph Laplacian matrix with quaternion-valued entries---$\ku$ is an archaic Greek letter often replaced by the Latin letter "q" as in "quaternion".
In this work, we extend the theory of spectral GCNs to digraphs with unrestricted edge weights also including digons. This is achieved by introducing the \emph{\laplaciano{}} $\symbolapl$, the first, to our knowledge, graph Laplacian matrix with quaternion-valued entries.\footnote{The letter $\ku$ in $\symbolapl$ is an archaic Greek letter often replaced in modern text by the Latin letter "q" as in "quaternion".}
We prove that $\symbolapl$ generalizes different previously proposed Laplacians (both real- and complex-valued), that it satisfies all the properties that are needed to build a convolutional operator around it, and that  it can naturally represent digons in such a way that the graph can be fully reconstructed from it without any loss of topological information.
Around the \laplaciano{}, we design \quater{}, a spectral GCN which relies on both
%our quaternion-valued Laplacian
$\symbolapl$ and quaternion-valued network weights.

While some GCNs employing quaternion weights have been recently proposed~\citep{nguyen2021quaternion,wang2022quaternion,le2023knowledge}, only QGNN, proposed by~\citet{nguyen2021quaternion}, is spectral based. As QGNN relies on the classical convolution operator of~\citet{kipf2016semi} bases on the classical (real-valued) graph Laplacian, it fails to fully capture the whole graph topology unless the graph has non-negative edge weights and is undirected (which implies that it contains no digons).

We summarize the differences between the \laplaciano{} $\symbolapl$ we introduce in this work and previous proposals in Table~\ref{tab:difference}.
\begin{table*}[t!]
    \centering
    % \small
\setlength{\tabcolsep}{2pt}
\caption{Main differences between Laplacian matrices used in spectral GCN literature. For a graph with edge set $E$, $w_1$ and $w_2$ denote the weights of the digons $(u, v), (v,u) \in E$.}
%\vspace{-.2cm}
\label{tab:difference}
\begin{tabular}{llcccccc}
\hline
Laplacian & & Number & Allows Negative & \multicolumn{4}{c}{Allows Digons? } \\
\cline{5-8}
Symbol & Proposed in & System & Edge Weights? & if $w_1 = w_2$ & if $w_1 = - w_2$ & if $w_1 \neq w_2$ & if $w_1 \neq -w_2$ \\
\hline
$L$ & \citet{kipf2016semi}                    &$\RR$ & \xmark     & \checkmark  & \xmark        & \xmark         & \xmark          \\
$L^{(q)}$ & \citet{zhang2021magnet}            & $\CC$ & \xmark     & \checkmark  & \xmark        & \xmark         & \xmark      \\
$L^\sigma$ & \citet{fiorini2022sigmanet}           & $\CC$ & \checkmark & \checkmark  & \xmark        & \xmark         & \xmark          \\
$L^{(q)}$ & \citet{he2022msgnn}            & $\CC$ & \checkmark & \checkmark  & \xmark        & \xmark         & \xmark          \\
$L^{(q)}$ & \citet{ko2022graph}            & $\CC$ & \checkmark & \checkmark  & \checkmark    & \xmark         & \xmark          \\
\bm{$\symbolapl$}   &  \textbf{This paper}   & $\HH$ & \checkmark & \checkmark  & \checkmark    & \checkmark     & \checkmark      \\ \hline
\end{tabular}
\end{table*}
For each Laplacian, the table reports the paper in which it was proposed or used within a spectral GCN and the corresponding symbol (which is often reused with some overload). It also indicates the type of number system associated with each Laplacian %of each 
and its capability to handle edges with negative weights and digons.
%
% We note that the table shows the symbols used in the original papers, which unfortunately are identical in three cases. To improve usability we have also given the reference to the paper in which they were proposed. 

%\paragraph{Main Contributions and Novelty of The Work}
\paragraph{Main Contributions of The Work}
\begin{itemize}
\item We propose the quaternion-valued Quaternionic Laplacian matrix $\symbolapl$, which naturally captures the presence of digons of different weight ({\em asymmetric}) without reducing them to a single edge as done in many previously-proposed Laplacians.
%such as the \textit{Magnetic Laplacian} and the \textit{Sign-Magnetic Laplacian. \sc{menzionarli nel cite dei papers e aggiungerli in tabella?}}
%
\item We prove that $\symbolapl$ generalizes both the standard Laplacian and the Sign-Magnetic Laplacian, as it coincides with the former when $G$ is undirected and with the latter when $G$ features directed edges with weights of arbitrary sign and all its digons have the same weight ({\em symmetric}).
%, but also 3) natively captures the presence of digons without transforming them into a single edge as done in, e.g., the \textit{Magnetic Laplacian} and the \textit{Sign-Magnetic Laplacian}.
% %
\item We incorporate $\symbolapl$ into \quater{}, a spectral-based GCN that includes quaternion-valued convolutional layers and quaternion-valued network weights.
\item Our experiments demonstrate that QuaterGCN consistently outperforms state-of-the-art spatial and spectral GCNs, particularly on tasks and datasets where the information carried by the digons is critical.

%\item Our experiments reveal that \quater{} outperforms almost every state-of-the-art spatial and spectral GCN on, in particular, tasks and datasets where the information the digons carry is crucial. %\st{To this end, we also introduce a new class of synthetic graphs that are parameterized by the percentage of digons they contain.}
%
%Experimental results show the superior performance of \quater{} compared to other state-of-the-art GCNs, particularly for tasks where the information on digons is crucial. For the purpose, we also introduce a new class of synthetic graphs parameterized by the percentage of digons they contain.
%
%we demonstrate experimentally that \quater{} outperforms other state-of-the-art graph convolutional networks, particularly for tasks that rely on digon information. To this end, we introduce a new class of synthetic graphs that are parameterized by the percentage of digons they contain. A more detailed description of the datasets used in this work can be found in the Appendix [cite]. All datasets used in this work are available in the GitHub repository, along with the code required to generate them.
\end{itemize}

%The remainder of the paper is organized as follows. Preliminaries and previous works are summarized in Section~\ref{sec:preliminaries}. The \textit{\laplaciano{}} operator is introduced in Section~\ref{sec:laplaciano} together with its properties. The section~\ref{sec:architecture} provides an overview of the \textit{\quater{}} architecture built on the \textit{\laplaciano{}}. Experimental results are reported in Section~\ref{sec:results}, where we test our model on node classification and edge prediction tasks, and compare it to different state-of-the-art spectral and spatial GCNs. Conclusions are drawn in Section~\ref{sec:conclusion}.

%The proofs of our theorems are provided in the Appendix.

\section{Preliminaries and Previous Works}\label{sec:preliminaries}

Let $G=(V,E)$ be an undirected graph with $n = |V|$ vertices without weights nor signs associated with its edges and let $A \in \{0,1\}^{n \times n}$ be its adjacency matrix. 
The {\em classical Laplacian matrix} $L \in \Z_+^{n \times n}$ of $G$ is defined as $L := D - A$, where $D := \diag(A \, e)$ is the degree matrix of $G$, $e$ is the all-one vector of appropriate size, and the operator $\diag$  builds a diagonal matrix with the argument on the main diagonal.
The normalized version of $L$ is defined as $L_\text{norm} := \rsqrt{D}(D - A)\rsqrt{D} = I - \rsqrt{D} A \rsqrt{D}$.

For a spectral convolution operator to be well-defined,  the graph Laplacian must fulfill three properties:
\begin{enumerate*}[label=\textit{\roman*)},series=MyList, before=\hspace{-0.3ex}]
    \item[\bf{P.1)}] it must be diagonalizable, i.e., it must admit an eigenvalue decomposition;
    \item[\bf{P.2)}] it must be positive semidefinite;
    \item[\bf{P.3)}] its spectrum must be upper-bounded by 2~\citep{kipf2016semi}.
\end{enumerate*}

While $L$ and $L_\text{norm}$ always satisfy such properties if $G$ is undirected (i.e., $A$ is symmetric) and has nonnegative edge weights (i.e., $A$ is component-wise nonnegative), this is not always the case for more general graphs.
%
% Notably, these matrices are rank deficient, symmetric, and positive semidefinite, ensuring that their eigenvalues are nonnegative. In the case of $L_\text{norm}$, the maximum eigenvalue is 2, a property that makes it particularly useful in convolutional networks (see further details below). \iln{MC:ma dopo lo diciamo perché? dove?}
%
%The definition in~\eqref{L-vanilla} still applies and the properties enjoyed by $L_\text{norm}$ still hold if $G$ features nonnegative weights i.e.,  $A_{ij} \geq 0$ for all $\{i,j\} \in E$. \iln{MC:Non capisco quest'ultima frase}
%
When $G$ is a digraph, $L$ is sometimes defined as a function of $A_s := \frac{1}{2}(A^\top + A)$  and $D_s := \Diag(A_s e)$, rather than of $A$ and $D$. This is, e.g., the case of QGNN~\citep{nguyen2021quaternion}. While such a choice preserves the mathematical properties of $L$, it significantly alters the topology of the graph. Indeed, when general graphs are considered, the following issues may arise:
\begin{itemize}
    %\item {\bf Issue 1.} If $G$ is a digraph, $A$ is, in general, not symmetric. This implies that $L$ is not symmetric and, thus, it may not admit an eigenvalue decomposition. \sc{con la simmetrizzata la direzione del grafo è morta}
    \item {\bf Issue 1.} If $G$ is a digraph, defining $L$ as a function of $A_s$ is equivalent to transforming the original graph into an undirected version of it, losing any directional information the original graph contained.
    %
    %If \textsc{Issue 2} arises and $A \notin \Z_+^{n \times n}$, $L$ may not be positive semidefinite even if $G$ is undirected and, also, $D$ may feature negative entries, which can lead to $\rsqrt{D}$ not being well-defined.
    \item {\bf Issue 2.} If $G$ features negative-weighted edges, neither $A$ nor $A_s$ belong, in general, to $\Z_+^{n \times n}$. This can lead to $D_{uu} < 0$ for some $u \in V$, in which case $L$ is not well defined in $\Z_+^{n \times n}$ due to $\rsqrt{D}$ being irrational.
    \item {\bf Issue 3.} If $G$ contains asymmetric digons, the weight asymmetry is completely lost in $A_s$, since the latter only features the average of the two weights. For example, in an author-citation graph, a digon representing two authors, the first one citing the second one 50 times while being cited by them only 2 times, would be identical to the two authors symmetrically citing each other precisely 26 times.
\end{itemize}

{\bf Issue 1} was first addressed by~\citet{zhang2021magnet, zhang2021smgc} by introducing Magnet, a spectral GCN relying on the \textit{Magnetic Laplacian} $L^{(q)}$. Such a Laplacian is a complex-valued extension of $L$ to unweighted directed graphs, and was originally proposed within an electro-magnetic charge model by~\citet{lieb1993fluxes}. It is defined as follows:
\begin{align*}
& L^{(q)} := D_s - H^{(q)},  \quad \text{with}\quad \\
& H^{(q)} := A_s \odot \exp \left(\ii \, \Theta^{(q)} \right), 
%\quad
& \Theta^{(q)} := 2 \pi q\left(A-A^\top \right),    
\end{align*}
where $\odot$ denotes the Hadamard product,
%$A_s := \frac{1}{2} \left(A+A^\top\right)$ is the symmetrized version of $A$,  $D_s := \Diag(A_s e)$,
$\ii = \sqrt{-1}$, $\Theta$ is a phase matrix that encodes the edge directions, and $\exp \left(\ii \, \Theta^{(q)} \right) := \cos(\Theta^{(q)}) + \ii \sin(\Theta^{(q)})$, where $\cos$ and $\sin$ are applied component-wise. 
The parameter $q \in \mathbb{R}_{0}^+$ is usually chosen in $\left[0, \frac{1}{4}\right]$ or $\left[0, \frac{1}{2}\right]$~\cite{zhang2021magnet,mag2017}. Choosing $q = 0$ implies $\Theta^{(q)} = 0$ and thus $L^{(q)}$ boils down to the Laplacian matrix $L$ defined on $A_s$ (in which case the directionality of $G$ is lost).

{\bf Issue 2} was first addressed by~\citet{fiorini2022sigmanet} via the introduction of the spectral GCN SigMaNet and the {\em Sign-Magnetic Laplacian} $L^{\sigma}$. Unlike $L^{(q)}$, $L^{\sigma}$ is well defined for graphs with negative edge weights and enjoys some robustness properties to weight scaling (which in $L^{(q)}$ could artificially alter the sign pattern of $\Theta$ and thus the directionality of the edges). $L^{\sigma}$ is defined as follows:
\small
\begin{align*}
& L^{\sigma} := \bar D_s - H^{\sigma}, \quad \text{with} \quad \\
& H^{\sigma} := A_s \odot \Big( ee^\top  - \sgn (|A - A^\top|) + \ii \sgn \big(|A| - |A^\top| \big) \Big), 
\end{align*}
\normalsize
where
%$A_s := \frac{1}{2} \left(A + A^\top \right)$,
$\bar D_s := \Diag(|A_s| \, e)$ and $\sgn: \RR \rightarrow \{-1,0,1\}$ is the {\em signum} function applied component-wise.
After~\citet{fiorini2022sigmanet}, \citet{he2022msgnn} extended the Magnetic Laplacian as a function of $\bar D_s := \Diag(\frac{|A| + |A|^\top}{2} \, e)$ rather than $D_s$. This Laplacian is well-defined for graphs with negative edge weights, but it still suffers from the scaling issue pointed out by \citet{fiorini2022sigmanet} which $L^{\sigma}$ avoids.

{\bf Issue 3}, to the best of our knowledge, has not yet been addressed. We set ourselves out to do so in this paper.

%\section{How About a Quaternion-Valued Laplacian?}\label{sec:laplaciano}
\section{A Quaternion-Valued Laplacian}\label{sec:laplaciano}

In this section, we introduce the \emph{\laplaciano{}} matrix, a positive semidefinite quaternion-valued Hermitian matrix which fully captures the directional and weight information of a digraph, even in the presence of digons, without restrictions on the sign or magnitude of the edge weights.

\subsection{The Quaternion Number System}

Quaternions are an extension of complex numbers to three imaginary components~\citep{hamilton1866elements}.
They are often used in quantum mechanics, where they lead to elegant expressions of the Lorentz transformation, which forms the basis of modern relativity theory~\citep{jia2008quaternions}. In computer graphics, quaternions are commonly used to represent and manipulate 3D objects for rotation estimation and pose graph optimization~\citep{carlone2015initialization}.

Formally, a quaternion $q \in \HH$ takes the form $q = q_0 + \ii q_1 + \jj q_2 + \kk q_3$, where $q_0, q_1, q_2$ and $q_3$ are real numbers and $\ii,\jj$, and $\kk$ are three imaginary units satisfying $\ii^2 = \jj^2 = \kk^2 = \ii\jj\kk = -1$. The four basis elements are $1$, $\ii$, $\jj$, and $\kk$. The conjugate of $q$ is $\bar q \equiv q^* = q_0 - \ii q_1 - \jj q_2 - \kk q_3$. $q$ is called imaginary if its real part $q_0$ is zero. The multiplication of quaternions satisfies the distribution law but is not commutative.
A quaternion-valued matrix $Q = (q_{uv}) \in \HH^{m\times n}$ reads $Q = Q_0 + \ii Q_1 + \jj Q_2 + \kk Q_3$, with $Q_0, Q_1, Q_2, Q_3 \in \RR^{m\times n}$. $Q^\top = (q_{vu})$ is the transpose of $Q$. $\bar Q = (\bar q_{uv})$ is the conjugate of $Q$. $Q^* = (\bar q_{vu}) = \bar Q^\top$ is the conjugate transpose of $Q$.
A square quaternion matrix $Q \in \HH^{n \times n}$ is called Hermitian if $Q^* = Q$ and skew-Hermitian if $Q^* = -Q$.
We denote the real part and the three imaginary parts of a quaternion $q \in \HH$ by $\Re(q)$, $\Im_1(q)$, $\Im_2(q)$, and $\Im_3(q)$.
%\sc{AUTOVALORI!?}

\subsection{The \laplaciano{}}

Let us now introduce the \laplaciano{}, which we denote by $\symbolapl$. First, we introduce the following matrices:
\begin{itemize}
    %ORIGINAL TOPOLOGY
    \item Let $T:= \sgn(|A|)\in \{0, 1\}^{n \times n}$ be a binary matrix that encodes the graph's topology, with $T_{uv} = 1$ if $G$ contains an edge from node $u$ to node $v$ and $T_{uv} = 0$ otherwise.
    %
    %ONLY DIGONS
    \item Let $O := T \odot T^\top\in \{0, 1\}^{n \times n}$ be the binary matrix that encodes the topology of the subgraph of $G$ that only contains digons (by definition, $O$ is symmetric and $O_{uv}= O_{vu}=1$ iff $T_{uv}= T_{vu}=1$).
    %
    %NO DIGONS (OF THE SAME WEIGHT)
    \item Let $N:= \sgn(|A - A^\top|) \in \{0, 1\}^{n \times n}$ be a binary matrix that encodes the topology of the subgraph of $G$ obtained by dropping any symmetric digons (by definition, $N_{uv} = 0$ if $A_{uv} = A_{vu}$ and $N_{uv}=T_{uv}$ otherwise).
    %
    %REDUCED TOPOLOGY (DIGONS ARE COLLAPES TO SINGLE EDGES)
    \item Let $R:= \sgn(|A| - |A^\top|) \in \{-1, 0, 1\}^{n \times n}$ be a signed binary matrix in which every asymmetric digon $(u,v),(v,u)$ is reduced to a single edge in the direction of the largest absolute weight ($R_{uv} = T_{uv}$ if $A_{uv} > A_{vu}$,  $R_{uv} = - T_{uv}$ if $A_{uv} < A_{vu}$, and $R_{uv} = 0$ otherwise).
\end{itemize}

With these definitions, we introduce the four  matrices $H^0, H^1, H^2, H^3$, whose elements are valued in \{-1, 0, 1\}:
$$
\begin{array}{ll}
\hspace{-0.2cm} H^0 := ee^\top  - N  &\hspace{-0.2cm}  H^1 := R \odot \big( e e^\top - O \big) \\
\hspace{-0.2cm}H^2 := O \odot N \odot \big(U(T) - L(T^\top) \big) &\hspace{-0.2cm}  H^3  := -H^2,
\end{array}
$$
where $U$ and $L$ are the unary operators that construct an upper- or lower-triangular matrix from the upper or lower triangle of the matrix given to them as input.

$H^0$ only encodes $G$'s symmetric digons, $H^1$ encodes all of $G$'s edges excluding digons, and $H^2$ and $H^3$ encode $G$'s asymmetric digons in a skew-symmetric way.
We remark that 
%(this will be crucial in the following) 
the three matrices $H^1, H^2, H^3$ are skew-symmetric by construction (i.e., $H^1_{uv} = -H^1_{vu}$, $H^2_{uv} = -H^2_{vu}$, $H^3_{uv} = -H^3_{vu}$ for all $u,v \in V$), whereas $H^0$ is symmetric.

Based on the four matrices $H^0, H^1, H^2, H^3$, we now define the \laplaciano{} as follows:
\begin{align}
\nonumber
& \symbolapl := \bar D_s - H^{\ku}, \quad \text{with} \quad \\
& H^{\ku} = A^1_s \odot (H^0 + \ii H^1) + \jj A_s^2 \odot  H^{2} + \kk A_s^3 \odot H^{3},
\end{align}
%\sout{where $A^1_s := \frac{|A| + |A|^\top}{2}$, $A^2_s := \frac{U(|A|) + L(|A|^\top)}{2}$, $A^3_s := \frac{L(|A|) + U(|A|^\top)}{2}$, and $\bar D_s := \Diag(|A^1_s| \, e)$.}
%
with $A^1_s := \frac{A + A^\top}{2}$, $A^2_s := \frac{U(A) + L(A^\top)}{2}$, $A^3_s := \frac{L(A) + U(A^\top)}{2}$, and %$\bar D_s := \frac{1}{2}\Diag(\left(|A| + |A^\top| \, e\right))$
$\bar D_s := \Diag(|A^1_s| \, e)$.
%For every $u,v$ with $1\leq u,v\in N$,  $U(A)_{uv}= A_{uv}\odot [u\leq v]$, where $[u\leq v]$ is the indicator function that evaluates to 1 if $u\leq v$ and to 0 otherwise. \red{Potremmo fare $U(A) := A \odot (\mathbb{1}_{u \leq v})_{uv \in N}$}.
%Similarly, the function $L$ extracts the lower-triangular portion of a matrix $A$, i.e., for every $u,v$ with $1\leq u,v\leq n$, it is defined as $L(A)_{uv}= A_{uv}\odot [u\geq v]$. 
Its normalized version reads:
\begin{equation}\label{eq:norma}
  L^{\ku}_\text{norm} := \rsqrt{\bar D_s} L^{\ku} \rsqrt{\bar D_s} = I - \rsqrt{\bar D_s} H^{\ku} \rsqrt{\bar D_s}.
\end{equation}
% Notice that, factoring by $\sgn \left( |A - A^T| \right)$, $H$ is equivalent to
% $$
% H = A_s + \sgn \left( |A - A^T| \right) \odot  \left( - A_s + \ii \left( U(A) - U(A)^T \right) + \jj \left( L(A) - L(A)^T \right)  \right).
% $$

\subsection{On the Nature of the Elements of $H^\ku$}

Let us illustrate how the topology of $G$ and its weights are mapped into the matrix $H^{\ku}$.

\begin{enumerate}
    \item For every edge $(u,v) \in E$ with $(v,u) \notin E$
    %that does not have an antiparallel edge $(v,u)$
    (i.e., not contained in a digon)
    %$H^{\ku}$ encodes the direction and weight of every
    $H^{\ku}_{uv} = - H^{\ku}_{vu} = 0 + \ii \frac{1}{2} A_{uv} + \jj0 + \kk0$ holds. Thus, the edge is purely mapped in the first imaginary component ($\ii$) of $H^{\ku}$.
    \item For every symmetric digon $(u,v),(v,u) \in E$ with $A_{uv} = A_{vu}$, $H^{\ku}_{uv} = H^{\ku}_{vu} = \frac{1}{2}(A_{uv} + A_{vu}) + \ii 0 + \jj 0 + \kk 0$ holds. Thus, such digons are encoded purely in the real part of $H^\ku$ (as if they coincided with an undirected edge).
    \item For every asymmetric digon $(u,v),(v,u) \in E$ with $A_{uv} \neq A_{vu}$, $H^{\ku}_{uv} = -H^{\ku}_{vu} = 0 + \ii 0 + \jj \frac{1}{2} A_{uv} - \kk \frac{1}{2} A_{vu}$ (if $u<v$) and $H^{\ku}_{uv} = -H^{\ku}_{vu} = 0 + \ii 0 - \jj \frac{1}{2} A_{uv} + \kk \frac{1}{2} A_{vu}$ (if $u>v$) hold.
    Such digons are thus encoded purely in the second and third ($\jj$ and $\kk$) imaginary components.
\end{enumerate}
%Appendix~\ref{appx:behaviuor}
%The Appendix contains a detailed example and a related discussion that further illustrate the individual contributions of the components of $\symbolapl$.

$L^\ku$ realizes the same mapping as $L^\sigma$ in cases 1 and 2, but not in case 3. In such a case, while in $L^\sigma$ every asymmetric digon is mapped into a single directed edge in the direction of largest magnitude (thus losing parts of it topology), in $L^\ku$ the topology of such a digon is completely preserved.
%\sc{coincide per scelta con sign-magn lapl anziche il mag lapl perche l'anno scorso abbiam fatto vedere che è superiore}

For each $u,v \in V$, the three imaginary parts of $L^\ku$ are orthogonal: if $u$ and $v$ share only one edge, this edge is reported in the imaginary component $\ii$; if they share two edges of different weights, the two edges are reported in the imaginary components $\jj$ and $\kk$; if they share two edges of the same weight, the edges are reported as a single undirected edge in the real component of $L^\ku$; if the two nodes share no edges, nothing is reported.

\subsection{From a Graph to its \laplaciano{}}

As an illustrative example, consider the graph depicted in Figure~\ref{fig:graph-matrix}, together with its weighted adjacency matrix.
The graph has the following characteristics:
\begin{enumerate}
    \item A single undirected edge $(1,2)$ (which is equivalent to a symmetric digon).
    \item A directed edge $(3,1)$.
    \item Two asymmetric digons, $(2,4)-(4,2)$ and $(3,4)-(4,3)$, with different weights that share the same sign.
\end{enumerate}
%
%The graph is chosen to highlight the behavior of the various components of the \laplaciano{} we proposed. 

\setlength\arraycolsep{2pt}
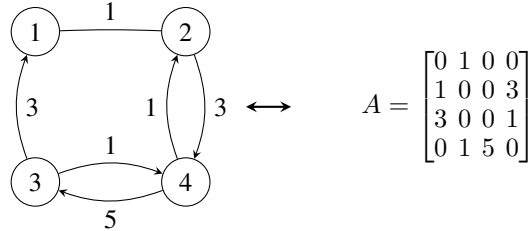
\begin{figure}[htb!]
\centering
\begin{tikzpicture}[>=stealth]
  \tikzset{vertex/.style = {circle, draw, fill=white, inner sep=3.5pt, minimum size=15pt}}
  
  % Define vertices
  \node[vertex] (v1) at (0,0) {1};
  \node[vertex] (v2) at (2,0) {2};
  \node[vertex] (v3) at (0,-2) {3};
  %\node[vertex] (v4) at (1,-3) {4};
  \node[vertex] (v4) at (2,-2) {4};
  
  % Draw edges
  \path[-] (v1) edge[bend left=2] node[above]{1} (v2);
  \path[->] (v3) edge[bend left=20] node[right]{3} (v1);
  \path[->] (v2) edge[bend left=20] node[right]{3} (v4);
  \path[->] (v4) edge[bend left=20] node[left]{1} (v2);
  \path[->] (v3) edge[bend left=20] node[above]{1} (v4);
  \path[->] (v4) edge[bend left=20] node[below]{5} (v3);
  %\path[->] (v4) edge[bend left=20] node[above]{2} (v2);
  %\path[->] (v5) edge[bend left=20] node[left]{2} (v2);
  %\path[->] (v5) edge[bend left=20] node[below]{-2} (v4);
  
  % Thick double-sided arrow
  \draw[thick, <->] (2.8,-1) -- (3.4,-1);
  
  % Adjacency matrix
  \node at (5.5, -1) {$ A= \begin{bmatrix}
  0 & 1 & 0 & 0 \\
  1 & 0 & 0 & 3  \\
  3 & 0 & 0 & 1 \\
  0 & 1 & 5 & 0 \\
  \end{bmatrix}$};
\end{tikzpicture}
\caption{Graph and Adjacency Matrix}
\label{fig:graph-matrix}
\end{figure}

For this graph, $\symbolapl$ reads:
\begin{align*}
& \symbolapl := \bar D_s - H^{\ku}, \quad \text{with}  \\
& H^{\ku} := A^1_s \odot (H^0 + \ii H^1) + \jj A_s^2 \odot  H^{2} + \kk A_s^3 \odot H^{3},
\end{align*}
with
$A^1_s~=\begin{bmatrix}
 0 & 1 & 1.5 & 0 \\
 1 & 0 & 0 & 2 \\
 1.5 & 0 & 0 & 3 \\
 0 & 2 & 3 & 0 \\
\end{bmatrix}$,
$A^2_s~=~\begin{bmatrix}
 0 & 0.5 & 0 & 0 \\
 0.5 & 0 & 0 & 1.5 \\
 0 & 0 & 0 & 0.5 \\
 0 & 1.5 & 0.5 & 0 \\
 \end{bmatrix}$, 
$A^3_s~=~\begin{bmatrix}
 0 & 0.5 & 1.5 & 0 \\
 0.5 & 0 & 0 & 0.5 \\
 1.5 & 0 & 0 & 2.5 \\
 0 & 0.5 & 2.5 & 0 \\
 \end{bmatrix}, \text{ and }
$ 

$H^0~=~\begin{bmatrix}
 1 & 1 & 0 & 1 \\
 1 & 1 & 1 & 0 \\
 0 & 1 & 1 & 0 \\
 1 & 0 & 0 & 1 \\
\end{bmatrix}$,
\hspace{1.2cm}$H^1~=~\begin{bmatrix}
 0 & -1 & 0 & 0 \\
 1 & 0 & 0 & 0 \\
 0 & 0 & 0 & 0  \\
 0 & 0 & 0 & 0 \\
\end{bmatrix}$,

$H^2~=~\begin{bmatrix}
 0 & 0 & 0 & 0 \\
 0 & 0 & 0 & 1 \\
 0 & 0 & 0 & 1 \\
 0 & -1 & -1 & 0 \\
\end{bmatrix}$,
\quad\quad$H^3~=~\begin{bmatrix}
 0 & 0 & 0 & 0 \\
 0 & 0 & 0 & -1 \\
 0 & 0 & 0 & -1 \\
0 & 1 & 1 & 0 \\
\end{bmatrix}.
$

We note that $H^0$ contains a 1 for each edge that does not belong to a digon with different weights or to a single directed edge without an antiparallel one. It features an all-1 diagonal, whose values would become zeros after the multiplication by $A_s^1$.
$H^1$ contains a $\pm1$ for each undirected edge with a positive sign for the actual edge direction and a negative one for the inverse direction.
$H^2$ and $H^3$ contain a 1 for each pair of edges belonging to a digon with different edge weights. By construction, $H^3$ is the opposite of $H^2$.

$\symbolapl$ thus the reads:
\begin{equation*}
%\scriptsize
    \symbolapl = 
    %\begin{bmatrix}
 % 2.5 & 0 & 0 & 0 & 0 \\
 % 0 & 3.5 & 0 & 0 & 0 \\
 % 0 & 0 & 2 & 0 & 0 \\
 % 0 & 0 & 0 & 2.5 & 0 \\
 % 0 & 0 & 0 & 0 & 4.5 \\
 % \end{bmatrix}         -
 % \begin{bmatrix}
 % 0 & 1 & -1.5\ii & 0 & 0 \\
 % 1 & 0 & 0 & 0 & 1.5\jj - 1\kk \\
 % 1.5\ii & 0 & 0 & 0.5\ii & 0 \\
 % 0 & 0 & -0.5\ii & 0 & 1\jj + 1\kk \\
 % 0 & -1.5\jj + 1\kk & 0 &  -1 - 1\jj & 0 \\
 % \end{bmatrix} =    
  \begin{bmatrix}
  2.5 & -1 & \ii \, 1.5 & 0 \\
  -1 & 3 & 0 & - \jj \, 1.5 + \kk \, 0.5 \\
  - \ii \, 1.5 & 0 & 4.5 & -\jj \, 0.5 + \kk \, 2.5 \\
  0 & \jj \, 1.5 - \kk \, 0.5 & \jj \, 0.5 - \kk \, 2.5 & 5
  \end{bmatrix}.
    \end{equation*}
By inspecting $\symbolapl$, one can observe that it encodes the elements of the graph in the following way:
\begin{enumerate}
\item The undirected edge is encoded via the real component $\symbolapl_{12} = \symbolapl_{21} = -1$;
\item The directed edge is encoded via the $\ii$ component,
$\symbolapl_{13} = \symbolapl_{31} = \ii \, 1.5$;
\item The two digons with different weights that share the same sign are encoded via the $\jj$ and $\kk$ components:
\begin{enumerate}
    \item $\symbolapl_{24} = - \symbolapl_{42} = - \jj \, 1.5 + \kk \, 0.5$;
    \item $\symbolapl_{34} = - \symbolapl_{43} = -\jj \, 0.5 + \kk \, 2.5$.
\end{enumerate}
%\item \textbf{Digon with different weights and opposite signs} via the $\jj$ and $\kk$ components, $\symbolapl_{34} = - \symbolapl_{43} = -\jj \, 0.5 + \kk \, 2.5$;
%\item \textbf{Digon with the same weights but opposite signs} via the $\jj$ and $\kk$ components, $\symbolapl_{54} = - \symbolapl_{45} = -\jj - \kk$.
\end{enumerate}

\subsection{Relationship between $\symbolapl$ and other Laplacians}

The \laplaciano{} is designed in such a way that it satisfies several desirable properties which we now illustrate.

For undirected graphs with positive edge weights, $\symbolapl$ generalizes the classical Laplacian $L$:
%SC: positive VS nonnegative edge weights: basta che ci sia un lato a peso positivo incidente in ogni nodo perchè $D$ sia ben def., no?
%SF: Vero, ma solo se il peso del lato positivo è maggiore (uguale?) alla somma dei lati incidenti con peso negativo
\begin{theorem}
$L^{\ku} = L$ for every graph with $A$ symmetric and nonnegative and $D_{vv} > 0$ for all $v \in V$.
\end{theorem}
For digraphs with arbitrary edge weight featuring, if any, symmetric digons, $\symbolapl$ generalizes the Sign-Magnetic Laplacian $L^\sigma$:
\begin{theorem}\label{thm:sigmanet}
$L^{\ku} = L^{\sigma}$ if, for all $u,v \in V$, either $A_{uv} = 0$ or $A_{uv} = A_{vu}$.
%\red{Only for graphs with strictly positive edge weights because of the different way we define $D$ w.r.t. the standard Laplacian and Sig-Ma-Net.}
\end{theorem}
As $L^\sigma$ coincides with $L^{(q)}$ with $q=\frac{1}{4}$~\citep{fiorini2022sigmanet}, the following holds:
\begin{corollary}
$\symbolapl= L^{(q)}$ with $q=\frac{1}{4}$ for every digraph with $A \in \{0, 1\}^{n \times n}$ containing, if any, symmetric digons.
%\red{Only for graphs with strictly positive edge weights because of the different way we define $D$ w.r.t. the standard Laplacian and Sig-Ma-Net.}
\end{corollary}

If the graph does not feature any symmetric digons, the matrix $H^{\ku}$ from which $\symbolapl$ is defined coincides with a linear combination of $L^\sigma$ with a Hermitian matrix encoding the asymmetric digons:
\begin{theorem}
Consider a weighted digraph without symmetric digons and let $H^m := \left( A_s^2 \odot H^2_1 + A_s^3 \odot H^3_1 \right)$, where $H^2_1 = N \odot \big(U(T) - L(T^\top) \big)$ and $H^3_1 = - H^2_1$. We have $H^0 = 0$ and $H^{\ku} = H^{\sigma} \odot \left( e e^\top - O \right) + O \odot H^m$. Thus, each component of $H^{\ku}_{uv}$ is a linear combination of the component $H^\sigma_{uv}$ of the Sign-Magnetic Laplacian $H^{\sigma}$ and the component $H^m_{uv}$ of the quaternionic Hermitian matrix $H^m$.
%Consider a weighted digraph without any digons of equal weight, let $H^m = \left( A_s^2 \odot H^2_1 + A_s^3 \odot H^3_1 \right)$, where $H^2_1 = N \odot \big(U(T) - L(T^\top) \big)$ and $H^3_1 = - H^2_1$. We have $H^{\ku} = H^{\sigma} \odot \left( e e^\top - O \right) + O \odot H^m$. Thus, each component of $H^{\ku}_{uv}$ is a linear combination of the component $H^\sigma_{uv}$ of the Sign-Magnetic Laplacian $H^{\sigma}$ and of the components $H^m_{uv}$ of the quaternion Hermitian matrix $H^m$.
%\red{Aggiungere $H^0$}
\end{theorem}

\subsection{Spectral properties of $\symbolapl$}

As $H^\ku$ is Hermitian by construction, $L^{\ku}$ and $L^{\ku}_\text{norm}$ are Hermitian as well. 
As any Hermitian quaternion matrix $Q$ is diagonalizable via the (right) eigenvalue decomposition $Q = U^* \Lambda U$ (see~\citet{qi2021note} for a proof), $L^{\ku}$ and $L^{\ku}_\text{norm}$ satisfy the property P.1.
Both matrices admit exactly $n$ right eigenvalue-eigenvector pairs, all of which are real:
\begin{theorem} \label{thm:psd}
$L^{\ku}$ and $L^{\ku}_{\text{norm}}$ are positive semidefinite.
\end{theorem}

%(we show this in the proof of Theorem~\ref{thm:psd}), 
The right eigenvalues of the normalized version of $L^{\ku}$ are upper bounded by 2:
\begin{theorem} \label{thm:eigenvalues}
$\lambda_\text{max}(L_{\text{norm}}^{\ku}) \leq 2$, where $\lambda_\text{max}$ denotes the (real) right eigenvalue of largest value.
\end{theorem}

These theorems show that $L^{\ku}$ and $L^{\ku}_\text{norm}$ enjoy the two remaining properties P.2 and P.3. Thus, $L^{\ku}_\text{norm}$ can be employed for the definition of a convolution operator, as shown in the following section.

%\paragraph{\red{Far vedere (e ri-derivare) cosa succede se swappiamo i due tringoli nella def del Laplaciano}}

\section{QuaterNet: a spectral GCN based on $\symbolapl$}\label{sec:architecture}

% \red{Consider a signal $x \in \RR^n$ and a filter $y \in \RR ^n$, the convolution operation is defined as $y * x = U \Diag(\widehat{y}) U^* x$. Letting $\Sigma := \Diag(\widehat{y})$, we define the matrix $Y := U \Sigma U^*$ as the {\em generalized convolution matrix}, such that $y * x = Yx$.
% %
% Let $\tilde \Lambda := \frac{2}{\lambda_\text{max}(Y)} \Lambda - I$ be a normalization of $\Lambda$. As $U U^* = I$, applying the same normalization to $Y$ results in $\tilde Y =  U \tilde \Lambda U^* = \frac{2}{\lambda_\text{max}} Y - I$.
% %
% Following~\citet{hammond2011wavelets,kipf2016semi}, we assume that $y$ is such that the entries of $\widehat{y}$ are real-valued polynomials in $\tilde \Lambda$, i.e., that $\widehat{y}_i = \sum_{k=0}^K \theta_k T_k (\tilde \lambda_i), i \in [n]$, where $\theta_0, \dots, \theta_K\in \RR$, $K \in \NN$, and $T_k$ is the Chebyshev polynomial of the first kind of order $k$'.
% %
% Assuming $\lambda_\text{max} \simeq 2$, we have $\tilde Y = \tilde M = M - I$. Letting $K = 1$ and $\theta_1 = -\theta_0$, the following equation represents the convolution operation underlying a spectral GCN:
% \begin{equation}\label{eq:Y-convolution}
%     y*x = Yx = (\theta_0 I - \theta_0 (M - I)) x = \theta_0 (2I - M)x.
% \end{equation}
% where $M \in \RR^{n \times n}$ is a matrix that satisfies the properties of being positive semidefinite and admitting the eigenvalue (or spectral) decomposition, $M = U \Lambda U^*$.}

Following~\citet{hammond2011wavelets} and \citet{kipf2016semi}, we define the convolution operation of a signal $x \in \RR^n$ and a filter $y \in \RR ^n$ as
\small
\begin{equation}\label{eq:Y-convolution}
    y*x = Yx = (\theta_0 I - \theta_0 (L^{\ku}_{\text{norm}} - I)) x = \theta_0 (2I - L^{\ku}_{\text{norm}})x.
\end{equation}
\normalsize
The equation is obtained by approximating the Fourier transform on the graph Laplacian with Chebyshev's polynomial (of the first kind) of order 1.
%SC: qua dovremmo dire che la Fourier transform è reale perchè lo spettro è reale, credo...

Due to Theorems~\ref{thm:psd} and \ref{thm:eigenvalues}, $L^{\ku}_{\text{norm}}$ enjoys properties P.1, P.2, and P.3. Thus, $Y$ is a well-defined convolution operator and, by definition of $L^{\ku}_{\text{norm}}$, we have:
%
% As previously stated, for the spectral convolution operator to be well defined the Laplacian matrix must satisfy properties~(P.1), (P.2), and (P.3).
% %
% As the hermiticity of $L^{\ku}_{\text{norm}}$, Theorem~\ref{thm:psd}, and Theorem~\ref{thm:eigenvalues} show that $L^{\ku}_{\text{norm}}$ enjoys these properties, we can set $M := L^{\ku}_\text{norm}$ and write the convolution operation as
\begin{align}\label{eq:beforetrick}
\nonumber
   y * x= Y x & = \theta_0 (2 I - (I - \rsqrt{\bar D_s} H^{\ku} \rsqrt{\bar D_s}) x \\
   & = \theta_0 (I + \rsqrt{\bar D_s} H^{\ku} \rsqrt{\bar D_s}) x.
\end{align}
Following~\citet{kipf2016semi} to avoid numerical instabilities, we apply Eq.~\eqref{eq:beforetrick} with $\rsqrt{\tilde D_s}\tilde H^{\ku} \rsqrt{\tilde D_s}$ rather than $I + \rsqrt{\bar D_s}H^{\ku} \rsqrt{\bar D_s}$, where $\tilde H^{\ku}$ and $\tilde D_{s}$ are defined as a function of %the modified adjacency matrix
$\tilde A := A + I$ rather than $A$.

We generalize the feature vector signal $x \in \HH^{n \times 1}$ to a feature matrix signal $X \in \HH^{n \times c}$ with $c$ input channels (i.e., a $c$-dimensional feature vector for every node of the graph).
Let $\Theta \in \HH^{c \times f}$ be a matrix representing the parameters of an $f$-dimensional filter and let $\phi$ be an activation function (such as the 
%rectified linear unit 
\textit{ReLU}) applied component-wise to the input matrix.

In \quater{}, we define the convolutional layer as the following mapping from $X$ to $Z^{\sigma}(X) \in \HH^{n \times f}$:
\begin{equation*}\label{eq:convolution}
%\nonumber 
Z^{\ku} (X) = \phi(\rsqrt{\tilde D_s}\tilde H^{\ku} \rsqrt{\tilde D_s}X\Theta) = \phi(X^\ku\Theta).
\end{equation*}
Since $X^\ku\Theta$ is a quaternion-valued matrix and traditional activation functions require a real argument, we follow the approach of~\citet{parcollet2018quaternion} and apply the activation function $\phi$ to each element of its quaternionic input as
\begin{equation*}
    \phi: (a + \ii b + \jj c + \kk d) \mapsto \phi(a) + \ii\phi(b) + \jj\phi(c) + \kk\phi(d).
\end{equation*}
Thanks to this, the output $Z^{\ku}(X)$ of the convolutional layer is quaternion-valued.
As usually done in spectral GCNs, we lastly adopt an \textit{unwind} layer by which we transform the matrix $Z^{\ku} (X)\in \HH^{n \times f}$ into the 4-times larger real-valued matrix
\small
\begin{equation*}
    \Big(\Re(Z^{\ku} (X)) \mid \Im_1(Z^{\ku} (X)) \mid \Im_2(Z^{\ku} (X)) \mid \Im_3(Z^{\ku} (X)) \Big) \in \RR^{n \times 4f}.
\end{equation*}
\normalsize

%While $\symbolapl$ and the various inputs and outputs to \quater{} can be encoded (and are in our code) via 4-times larger real-valued tensors/matrices, the algebraic operations that are carried out in \quater{} between its quaternion tensors/matrices via the Hamilton product cannot be directly obtained via the standard algebra adopted for their real-valued counterparts.
%
%The adoption of quaternion algebra offers an advantage whereby the various components of the quaternion $\Theta$ are distributed among the four components of $X^\ku$.
As stated by~\citet{parcollet2020survey} and \citet{nguyen2021quaternion}, the adoption of quaternion algebra in the product $X^\ku \Theta$ leads to an extensive interaction among the components of $X^\ku$ which is likely to lead to more expressive vector representations than those achieved with real- and complex-valued 
%neural 
networks.

%Therefore, even a small change in the values multiplied by $\Theta$ can result in a significantly different output, leading to distinct performance outcomes. This phenomenon is a crucial reason why quaternion space allows for highly expressive computations via the Hamilton product, as compared to Euclidean and complex vector spaces~\cite{parcollet2020survey}. Additionally, this iteration process enables the model to learn the potential relationships within and between each hidden layer, thereby enhancing the quality of the representation.

Based on the task at hand (see next section), after the convolution layer we described before \quater{} features either a linear layer with weights $W$ or a $1D$ convolution. 
Considering, for example, a node-classification task of predicting which of a set of unknown classes a graph vertex belongs to, \quater{} implements the function
\small
\begin{equation*}
    \text{softmax} \left(\text{unwind}\left(Z^{\ku(2)} \left(Z^{\ku(1)} \left(X^{(0)} \right)\right)\right)       W \right),
\end{equation*}
\normalsize
where $X^{(0)} \in \HH ^{n \times c}$ is the input feature matrix, $Z^{\ku(1)} \in \HH^{n \times f_1}$ and $Z^{\ku(2)} \in \HH^{n \times f_2}$ are two convolutional layers, $W \in \RR^{4f_2 \times d}$ are the weights of the linear layer (with $d$ being the number of classes), and $\text{softmax}: \RR^{n \times d} \rightarrow [0, 1]^{n \times d}$ is the normalized exponential activation function typically used to recover the node classes.

\subsection{Complexity of \quater{}}
%\iln{se mettessimo questa sezione alla fine? parliamo di task che non abbiamo presentato}
%SC: incasina, sadly
For a graph with $n$ nodes with a $c$-dimensional feature vector each, the complexity of \quater{}, with two convolutional layers with $f_1$ and $f_2$ filters each, is $O\left(nc(n+f_1)+n f_1(n+f_2)+m^\text{train} f_2 d\right)$ for an edge-classification task and $O \left(nc(n+f_1)+n f_1(n+f_2)+n f_2 d\right)$ for a node-classification one, where $d$ is the number of classes (see the next section for more details on these tasks). Such a complexity is quadratic in the number of nodes $n$ and it coincides with MagNet's, MSGNN's, and SigMaNet's.
As the scalar-scalar product between two quaternions requires 8 multiplications and 7 additions between real numbers rather than 4 multiplications and 3 additions for the complex case and a single multiplication for the real case, the least-efficient operation carried out by \quater{} can be up to 64 times slower than it would be if the network featured real-only numbers. While apparently large, such a quantity is constant w.r.t. the size of the graph, which implies that \quater{} scales comparably to previous proposals in the literature.

% We remark that, while enjoying a wider applicability due to being able to handle graphs with edge weights unrestricted in sign), \quater{} features half the weights of MagNet---\red{this, as well as the impact of the architectural differences between the two on graphs with nonnegative weights, are better explained in the appendix. \iln{add 3 righe di sunto dell'ultima sez. dell'appendix}}

\section{Numerical Experiments}\label{sec:results}

We compare \quater{} with state-of-the-art GCNs across four tasks: \textit{node classification} (NC), \textit{three-class edge prediction} (3CEP), \textit{four-class edge prediction} (4CEP), and \textit{five-class edge prediction} (5CEP).
Throughout this section, the tables report the best results in \textbf{boldface} and the second-best are \underline{underlined}. The datasets and the code we used are publicly available
%on GitHub.\footnote{\url{https://github.com/Stefa1994/QuaterGCN}}
at \url{https://github.com/Stefa1994/QuaterGCN}.

We experiment on the six widely-used real-world directed graphs {\tt Bitcoin-OTC}, {\tt Bitcoin Alpha}, {\tt WikiRfa}, {\tt Telegram}, {\tt Slashdot}, and {\tt Epinions} (see~\citet{7837846, bovet2020activity, west2014exploiting, leskovec2010signed}).
%These datasets are directed graphs,.
The first three {feature} edge weights of unrestricted sign and magnitude; the fourth contains graphs with positive edge weights, while the last two have graphs with weights satisfying $A \in \{-1,0, +1\}^{n \times n}$.
%In these five datasets, the classes of positive and negative weighted edges are imbalanced (i.e., nearly 80\% are positive edges).
%Questi dataset sono tutti direzionati
%{\tt WikiRfa}~\citep{west2014exploiting}

%To assess \quater{}'s performance as the graph density increases, 
To study the relationship between performance and graph density, we also employ {\tt DBSM} graphs: synthetic digraphs with positive random weights already used by \citet{fiorini2022sigmanet}. They are generated via a direct stochastic block model (DSBM) with edge weights greater than 1 and a different number of nodes per cluster ($N$) and number of clusters ($C$). Inter- and intra-cluster edges are created with, respectively, probability $\alpha_{uv}$ and $\alpha_{uu}$, and a connected pair of nodes $\{u,v\}$ with $u < v$ shares the edge $(u,v)$ with probability $\beta_{uv}$ and $(v,u)$ with probability $1-\beta_{uv}$.
%SC: per simmetria dovremmo dire ogni grafo che parametri ha (lo facciamo per quelli coi digons che seguono...)

As the {\tt DBSM} graphs are digons-free, we introduce a second class of synthetic digraphs with a variable percentage of digons $\delta \in (0, 1)$ with positive random weights between 2 and 4: {\tt Di150} (150 nodes) and {\tt Di500} (500 nodes).
%
%Moreover, we explore three distinct digon schemes: \textit{i)} random weights ranging from 2 to 4, \textit{ii)} opposite weights (where one weight is the negative of the other), or \textit{iii)} equal weights (yielding an undirected edge). For the conducted experiment, the synthetic graphs were generated using the first scheme.}
{\tt Di150} features graphs with $N = 150$, $C = 5$, $\alpha_{uu} = 0.1$, $\beta_{uv} = 0.2$, and $\alpha_{uv} = 0.6$.
{\tt Di500} contains graphs with $N = 500$, $C = 5$, $\alpha_{uu} = 0.1$, $\beta_{uv} = 0.2$, and $\alpha_{uv} = 0.1$. %in both cases the percentage of digons $\delta$ is in $\{0.2, 0.5, 0.7\}$.
%}
Notice that {\tt Di500} is sparser than {\tt Di250}.

\begin{table*}[htb!]
\setlength{\tabcolsep}{2pt}
\footnotesize
\centering
\caption{Accuracy (\%) on the node classification task}
%\vspace{-.2cm}
\label{tab:node}
\begin{tabular}{llllllllHHH}
\hline
 & \multicolumn{10}{c}{Node classification} \\ \cline{2-11} 
 &  & \multicolumn{3}{c}{{\tt DBSM}} & \multicolumn{3}{c}{{\tt Di500}} & \multicolumn{3}{H}{{\tt Di150}} \\ \cline{2-11} 
 & Telegram & $\alpha_{uv}$ = 0.05 & $\alpha_{uv}$ = 0.08 & $\alpha_{uv}$ = 0.1 & $\delta$ = 0.2 & $\delta$ = 0.5 & $\delta$ = 0.7  & $\delta$ = 0.2 & $\delta$ = 0.5 & $\delta$ = 0.7 \\ \midrule
ChebNet & 61.73$\pm$4.25 & 20.06$\pm$00.18 & 20.50$\pm$00.77 & 19.98$\pm$00.06 & 19.90$\pm$0.24 & 20.00$\pm$0.00 & 19.94$\pm$0.13 & 19.93$\pm$00.20 & 20.00$\pm$00.00 & 20.00$\pm$00.00 \\
GCN & 60.77$\pm$3.67 & 20.06$\pm$00.18 & 20.02$\pm$00.06 & 20.01$\pm$00.01 & 20.04$\pm$0.12 & 20.10$\pm$0.30 & 20.08$\pm$0.30  & 20.00$\pm$00.00 & 20.00$\pm$00.00 & 20.07$\pm$00.20 \\
QGNN & 51.35$\pm$9.10 & 20.03$\pm$00.12 & 20.01$\pm$00.02 & 19.99$\pm$00.07 & 20.23$\pm$0.21 & 19.94$\pm$0.18 & 20.00$\pm$0.00 & 19.87$\pm$00.40 & 19.93$\pm$00.20 & 20.00$\pm$00.00 \\ \midrule
APPNP & 55.19$\pm$6.26 & 33.46$\pm$07.43 & 34.72$\pm$14.98 & 36.16$\pm$14.92 & 20.64$\pm$1.32 & 20.16$\pm$0.37 & 20.10$\pm$0.30 & 22.87$\pm$08.60 & 21.07$\pm$03.20 & 20.00$\pm$00.00  \\
SAGE & 65.38$\pm$5.15 & 67.64$\pm$09.81 & 68.28$\pm$10.92 & 82.96$\pm$10.98 & 23.68$\pm$3.83 & 20.44$\pm$0.95 & 20.02$\pm$0.14 & 78.40$\pm$14.35 & 33.20$\pm$14.10 & 20.33$\pm$01.69 \\
GIN & {72.69${\pm}$4.62} & 28.46$\pm$08.01 & 20.12$\pm$00.20 & 20.98$\pm$08.28 & 20.14$\pm$0.42 & 19.88$\pm$0.36 & 20.00$\pm$0.00 & 24.67$\pm$06.76 & 28.13$\pm$08.87 & 23.67$\pm$06.24\\
GAT & 72.31$\pm$3.01 & 22.34$\pm$03.13 & 21.90$\pm$02.89 & 21.58$\pm$01.80 & 19.90$\pm$0.20 & 20.04$\pm$0.28 & 20.16$\pm$0.42 & 51.20$\pm$10.18 & 29.33$\pm$10.05 & 21.40$\pm$03.78 \\
SSSNET & 24.04$\pm$9.29 & \textbf{91.04$\bm{\pm}$03.60} & {94.94$\pm$01.01} & {96.77$\pm$00.80} & 31.41$\pm$5.91 & 22.34$\pm$1.31 & \underline{21.13$\pm$1.03} & 92.71$\pm$12.48 & 86.18$\pm$17.77 & \underline{61.78$\pm$24.44} \\ \midrule
DGCN & 71.15$\pm$6.32 & 30.02$\pm$06.57 & 30.22$\pm$11.94 & 28.40$\pm$08.62 & 20.10$\pm$0.30 & 20.00$\pm$0.00 & 20.00$\pm$0.00 & 26.87$\pm$08.43 & 21.87$\pm$05.60 & 20.00$\pm$00.00\\
DiGraph & 71.16$\pm$5.57 & 53.84$\pm$14.28 & 38.50$\pm$12.20 & 34.78$\pm$09.94 & \underline{32.82$\pm$2.14} & \underline{24.44$\pm$2.33} & 20.76$\pm$1.34 & 97.40$\pm$01.01 & \underline{88.27$\pm$03.14} & {58.93$\pm$08.47}  \\
DiGCL & 64.62$\pm$4.50 & 19.51$\pm$01.21 & 20.24$\pm$00.84 & 19.98$\pm$00.45 & {20.00$\pm$0.00} & 20.00$\pm$0.00 & 20.00$\pm$0.00 & 20.00$\pm$00.00 & 20.00$\pm$00.00 & 20.00$\pm$00.00\\
MagNet & 55.96$\pm$3.59 & 78.64$\pm$01.29 & 87.52$\pm$01.30 & 91.58$\pm$01.04 & {31.46$\pm$2.20} & 22.74$\pm$1.12 & {20.88$\pm$1.62} & \underline{97.87$\pm$01.90} & 74.53$\pm$10.43 & 31.13$\pm$06.65 \\
\signum{} & \underline{74.23$\bm{\pm}$5.24} & {87.44$\pm$00.99} & \underline{96.14${\pm}$00.64} & \underline{98.60${\pm}$00.31} & 31.26$\pm$2.08 & 22.32$\pm$1.69 & 19.94$\pm$1.07 & 74.67$\pm$04.15 & 49.60$\pm$03.07 & 24.67$\pm$04.40 \\ \midrule
Gra\quater{} & \textbf{75.58$\pm$3.85} & \underline{87.46$\pm$00.73} & \textbf{96.44$\bm{\pm}$00.12} & \textbf{98.80$\bm{\pm}$00.20} & \textbf{64.28$\bm{\pm}$1.04} & \textbf{70.60$\bm{\pm}$1.62} & \textbf{71.58$\bm{\pm}$1.52} & \textbf{99.73$\bm{\pm}$00.33} & \textbf{99.85$\bm{\pm}$00.13} & \textbf{99.93$\bm{\pm}$00.03}
\end{tabular}
\end{table*}

\begin{table}[htb!]
\setlength{\tabcolsep}{2pt}
\footnotesize
\centering
\caption{Accuracy (\%) on the node classification task}
%\vspace{-.2cm}
\label{tab:node1}
\begin{tabular}{lHHHHHHHlll}
\hline
 & \multicolumn{10}{c}{Node classification} \\ \cline{2-11} 
 &  & \multicolumn{3}{H}{{\tt DBSM}} & \multicolumn{3}{H}{{\tt Di500}} & \multicolumn{3}{c}{{\tt Di150}} \\ \cline{2-11} 
 & Telegram & $\alpha_{uv}$ = 0.05 & $\alpha_{uv}$ = 0.08 & $\alpha_{uv}$ = 0.1 & $\delta$ = 0.2 & $\delta$ = 0.5 & $\delta$ = 0.7  & $\delta$ = 0.2 & $\delta$ = 0.5 & $\delta$ = 0.7 \\ \midrule
ChebNet & 61.73$\pm$4.25 & 20.06$\pm$00.18 & 20.50$\pm$00.77 & 19.98$\pm$00.06 & 19.90$\pm$0.24 & 20.00$\pm$0.00 & 19.94$\pm$0.13 & 19.93$\pm$00.20 & 20.00$\pm$00.00 & 20.00$\pm$00.00 \\
GCN & 60.77$\pm$3.67 & 20.06$\pm$00.18 & 20.02$\pm$00.06 & 20.01$\pm$00.01 & 20.04$\pm$0.12 & 20.10$\pm$0.30 & 20.08$\pm$0.30  & 20.00$\pm$00.00 & 20.00$\pm$00.00 & 20.07$\pm$00.20 \\
QGNN & 51.35$\pm$9.10 & 20.03$\pm$00.12 & 20.01$\pm$00.02 & 19.99$\pm$00.07 & 20.23$\pm$0.21 & 19.94$\pm$0.18 & 20.00$\pm$0.00 & 19.87$\pm$00.40 & 19.93$\pm$00.20 & 20.00$\pm$00.00 \\ \midrule
APPNP & 55.19$\pm$6.26 & 33.46$\pm$07.43 & 34.72$\pm$14.98 & 36.16$\pm$14.92 & 20.64$\pm$1.32 & 20.16$\pm$0.37 & 20.10$\pm$0.30 & 22.87$\pm$08.60 & 21.07$\pm$03.20 & 20.00$\pm$00.00  \\
SAGE & 65.38$\pm$5.15 & 67.64$\pm$09.81 & 68.28$\pm$10.92 & 82.96$\pm$10.98 & 23.68$\pm$3.83 & 20.44$\pm$0.95 & 20.02$\pm$0.14 & 78.40$\pm$14.35 & 33.20$\pm$14.10 & 20.33$\pm$01.69 \\
GIN & {72.69${\pm}$4.62} & 28.46$\pm$08.01 & 20.12$\pm$00.20 & 20.98$\pm$08.28 & 20.14$\pm$0.42 & 19.88$\pm$0.36 & 20.00$\pm$0.00 & 24.67$\pm$06.76 & 28.13$\pm$08.87 & 23.67$\pm$06.24\\
GAT & 72.31$\pm$3.01 & 22.34$\pm$03.13 & 21.90$\pm$02.89 & 21.58$\pm$01.80 & 19.90$\pm$0.20 & 20.04$\pm$0.28 & 20.16$\pm$0.42 & 51.20$\pm$10.18 & 29.33$\pm$10.05 & 21.40$\pm$03.78 \\
SSSNET & 24.04$\pm$9.29 & \textbf{91.04$\bm{\pm}$03.60} & {94.94$\pm$01.01} & {96.77$\pm$00.80} & 31.41$\pm$5.91 & 22.34$\pm$1.31 & \underline{21.13$\pm$1.03} & 92.71$\pm$12.48 & 86.18$\pm$17.77 & \underline{61.78$\pm$24.44} \\ \midrule
DGCN & 71.15$\pm$6.32 & 30.02$\pm$06.57 & 30.22$\pm$11.94 & 28.40$\pm$08.62 & 20.10$\pm$0.30 & 20.00$\pm$0.00 & 20.00$\pm$0.00 & 26.87$\pm$08.43 & 21.87$\pm$05.60 & 20.00$\pm$00.00\\
DiGraph & 71.16$\pm$5.57 & 53.84$\pm$14.28 & 38.50$\pm$12.20 & 34.78$\pm$09.94 & \underline{32.82$\pm$2.14} & \underline{24.44$\pm$2.33} & 20.76$\pm$1.34 & 97.40$\pm$01.01 & \underline{88.27$\pm$03.14} & {58.93$\pm$08.47}  \\
DiGCL & 64.62$\pm$4.50 & 19.51$\pm$01.21 & 20.24$\pm$00.84 & 19.98$\pm$00.45 & {20.00$\pm$0.00} & 20.00$\pm$0.00 & 20.00$\pm$0.00 & 20.00$\pm$00.00 & 20.00$\pm$00.00 & 20.00$\pm$00.00\\
MagNet & 55.96$\pm$3.59 & 78.64$\pm$01.29 & 87.52$\pm$01.30 & 91.58$\pm$01.04 & {31.46$\pm$2.20} & 22.74$\pm$1.12 & {20.88$\pm$1.62} & \underline{97.87$\pm$01.90} & 74.53$\pm$10.43 & 31.13$\pm$06.65 \\
\signum{} & \underline{74.23$\bm{\pm}$5.24} & {87.44$\pm$00.99} & \underline{96.14${\pm}$00.64} & \underline{98.60${\pm}$00.31} & 31.26$\pm$2.08 & 22.32$\pm$1.69 & 19.94$\pm$1.07 & 74.67$\pm$04.15 & 49.60$\pm$03.07 & 24.67$\pm$04.40 \\ \midrule
\quater{} & \textbf{75.58$\pm$3.85} & \underline{87.46$\pm$00.73} & \textbf{96.44$\bm{\pm}$00.12} & \textbf{98.80$\bm{\pm}$00.20} & \textbf{64.28$\bm{\pm}$1.04} & \textbf{70.60$\bm{\pm}$1.62} & \textbf{71.58$\bm{\pm}$1.52} & \textbf{99.73$\bm{\pm}$00.33} & \textbf{99.85$\bm{\pm}$00.13} & \textbf{99.93$\bm{\pm}$00.03}
\end{tabular}
\end{table}

\begin{table*}[t!]
\setlength{\tabcolsep}{2pt}
\footnotesize
\centering
\caption{Accuracy (\%) on the three-class edge prediction task}
%\vspace{-.2cm}
\label{tab:threeclass}
\begin{tabular}{llllllllll}
\hline
 & \multicolumn{9}{c}{Three-Class Edge prediction} \\ \cline{2-10} 
\multicolumn{1}{l}{} &  &  &  & \multicolumn{3}{c}{{\tt Di150}} & \multicolumn{3}{c}{{\tt DBSM}} \\ \cline{2-10} 
 & Telegram & Bit Alpha* & Bitcoin OTC* & $\delta$ = 0.2 & $\delta$ = 0.5 & $\delta$ = 0.7  & $\alpha_{uv}= 0.05$ & $\alpha_{uv}= 0.08$ & $\alpha_{uv}= 0.1$ \\ \hline
ChebNet & 63.65$\pm$4.65 & 82.82$\pm$0.82 & 83.01$\pm$1.09 & 40.58$\pm$0.07 & 48.81$\pm$0.53 & 52.10$\pm$0.88 & 33.34$\pm$0.01 & 33.36$\pm$0.07 & 33.33$\pm$0.02 \\
GCN & 53.86$\pm$1.60 & 82.61$\pm$0.67 & 82.49$\pm$0.99 & 40.60$\pm$0.07 & \underline{49.00$\pm$0.08} & \textbf{53.51$\bm{\pm}$0.10} & 33.33$\pm$0.02 & 33.32$\pm$0.03 & 33.34$\pm$0.01 \\
QGNN & 52.33$\pm$1.50 & 80.93$\pm$0.63 & 79.97$\pm$0.80 & 40.60$\pm$0.07 & \underline{49.00$\pm$0.08} & 52.51$\pm$0.09 & 33.38$\pm$0.09 & 33.43$\pm$0.26 & 33.37$\pm$0.04 \\ \midrule
APPNP & 50.82$\pm$6.31 & 82.14$\pm$0.89 & 81.77$\pm$0.63 & 40.55$\pm$0.08 & 48.98$\pm$0.09 & 52.50$\pm$0.10 & 37.67$\pm$4.04 & 37.84$\pm$5.70 & 37.52$\pm$5.33 \\
SAGE & 69.28$\pm$7.24 & 55.82$\pm$1.60 & 85.19$\pm$0.64 & 40.62$\pm$0.14 & \underline{49.00$\pm$0.08} & \underline{52.52$\pm$0.09} & 39.50$\pm$3.74 & 38.51$\pm$3.55 & 42.69$\pm$3.87 \\
GIN & 58.41$\pm$1.26 & 77.93$\pm$0.86 & 76.35$\pm$0.77 & 40.56$\pm$0.08 & 48.98$\pm$0.10 & 52.47$\pm$0.10 & 34.65$\pm$2.62 & 33.34$\pm$0.01 & 33.52$\pm$0.37 \\
GAT & 67.34$\pm$2.50 & 84.93$\pm$1.20 & 85.02$\pm$0.74 & 40.64$\pm$1.79 & 48.44$\pm$1.63 & 52.51$\pm$0.09 & 33.70$\pm$0.79 & 33.35$\pm$0.07 & 33.91$\pm$1.63 \\ \midrule
DGCN & 75.01$\pm$3.60 & 85.01$\pm$0.95 & 85.03$\pm$0.64 & 40.57$\pm$0.06 & 48.82$\pm$0.50 & 52.51$\pm$0.08 & 34.12$\pm$2.17 & 34.78$\pm$2.11 & 35.24$\pm$2.36 \\
DiGraph & 74.27$\pm$1.02 & 83.66$\pm$0.72 & 84.14$\pm$0.82 & 41.38$\pm$0.92 & 48.90$\pm$0.11 & 52.40$\pm$0.09 & 41.30$\pm$1.41 & 42.57$\pm$1.62 & 53.57$\pm$1.73 \\
DiGCL & 66.03$\pm$0.84 & 77.68$\pm$0.74 & 76.35$\pm$0.77 & 29.70$\pm$0.04 & 25.50$\pm$0.04 & 23.74$\pm$0.04 & 38.30$\pm$0.15 & 38.17$\pm$0.07 & 37.58$\pm$0.12 \\
MagNet & \textbf{82.28$\bm{\pm}$0.84} & \underline{85.72$\pm$0.67} & \underline{85.66$\pm$0.78} & 45.47$\pm$1.70 & 48.78$\pm$0.35 & 52.19$\pm$0.43 & 43.62$\pm$1.11 & \underline{46.76$\pm$1.13} & {47.76$\pm$1.12} \\ 
\signum{} & {80.13$\pm$0.87} & {85.52$\pm$0.61} & {84.61$\pm$0.79} & \underline{45.50$\pm$1.41} & 47.02$\pm$0.91 & 51.81$\pm$0.80 & \underline{43.65$\pm$0.36} & \textbf{47.26$\bm{\pm}$0.17} & \underline{48.60$\pm$0.17} \\ \midrule
\quater{} & \underline{81.17$\pm$0.74} & \textbf{86.17$\bm{\pm}$0.57} & \textbf{86.06$\bm{\pm}$0.60} & \textbf{47.14$\bm{\pm}$0.21} & \textbf{49.01$\bm{\pm}$0.16} & 52.07$\pm$0.22 & \textbf{44.10$\bm{\pm}$0.58} & \textbf{47.26$\bm{\pm}$0.56} & \textbf{48.68$\bm{\pm}$0.26}
\end{tabular}
\end{table*}

%\begin{table}[t!]
%\setlength{\tabcolsep}{2pt}
%%\footnotesize
%\centering
%\begin{tabular}{llll}
%\hline
% & \multicolumn{3}{c}{Three Class Edge prediction} \\ \cline{2-4} 
%\multicolumn{1}{l}{} & \multicolumn{3}{c}{{\tt DBSM}} \\ \cline{2-4} 
% & $\alpha_{uv}= 0.05$ & $\alpha_{uv}= 0.08$ & $\alpha_{uv}= 0.1$ \\ \hline
%ChebNet & 33.34$\pm$0.01 & 33.36$\pm$0.07 & 33.33$\pm$0.02 \\
%GCN & 33.33$\pm$0.02 & 33.32$\pm$0.03 & 33.34$\pm$0.01 \\
%QGNN & 33.38$\pm$0.09 & 33.43$\pm$0.26 & 33.37$\pm$0.04 \\ \hline
%APPNP & 37.67$\pm$4.04 & 37.84$\pm$5.70 & 37.52$\pm$5.33 \\
%SAGE & 39.50$\pm$3.74 & 38.51$\pm$3.55 & 42.69$\pm$3.87 \\
%GIN & 34.65$\pm$2.62 & 33.34$\pm$0.01 & 33.52$\pm$0.37 \\
%GAT & 33.70$\pm$0.79 & 33.35$\pm$0.07 & 33.91$\pm$1.63 \\ \hline
%DGCN & 34.12$\pm$2.17 & 34.78$\pm$2.11 & 35.24$\pm$2.36 \\
%DiGraph & 41.30$\pm$1.41 & 42.57$\pm$1.62 & 53.57$\pm$1.73 \\
%DiGCL & 38.30$\pm$0.15 & 38.17$\pm$0.07 & 37.58$\pm$0.12 \\
%MagNet & 43.62$\pm$1.11 & \underline{46.76$\pm$1.13} & {47.76$\pm$1.12} \\
%\signum{} & \underline{43.65$\pm$0.36} & \textbf{47.26$\bm{\pm}$0.17} & \underline{48.60$\pm$0.17} \\ \midrule
%\quater{} & \textbf{44.10$\bm{\pm}$0.58} & \textbf{47.26$\bm{\pm}$0.56} & \textbf{48.68$\bm{\pm}$0.26}
%\end{tabular}
%\end{table}

\subsection{Node Classification Task (NC)}

\begin{table*}[htb!]
\setlength{\tabcolsep}{2pt}
\footnotesize
\centering
\caption{Accuracy (\%) on the four-class edge prediction tasks}
%\vspace{-.2cm}
\label{tab:four-class}
\begin{tabular}{lrrrrrHHHHH}
\hline
\multicolumn{1}{c}{} & \multicolumn{5}{c}{Four-Class Edge prediction} & \multicolumn{4}{H}{Five-Class Edge prediction}
%\multicolumn{1}{c}{} & \multicolumn{5}{c}{Four Class Edge prediction} %\\ \cline{2-6}
%\multicolumn{1}{c}{} & \multicolumn{4}{c}{Five Class Edge prediction} 
\\ \cline{2-11} 
 & \multicolumn{1}{c}{Bitcoin Alpha} & \multicolumn{1}{c}{Bitcoin OTC} & \multicolumn{1}{c}{WikiRfa} & \multicolumn{1}{c}{Slashdot} & \multicolumn{1}{c}{Epinions}   & \multicolumn{1}{H}{Bitcoin Alpha} & \multicolumn{1}{H}{Bitcoin OTC} & \multicolumn{1}{H}{WikiRfa} & \multicolumn{1}{H}{Slashdot} & \multicolumn{1}{H}{Epinions} \\ \hline
SGCN & 48.05$\pm$0.29 & 52.52$\pm$0.71 & 68.37$\pm$0.51 & 64.01$\pm$0.24 &  67.99$\pm$0.56 & 78.43$\pm$0.36 & 77.54$\pm$0.56 & 67.74$\pm$0.29 & 64.74$\pm$0.16 &  74.07$\pm$0.32\\
SiGAT & 50.12$\pm$1.80 & 50.86$\pm$1.45 & 57.68$\pm$0.63 &54.82$\pm$0.32  & 60.21$\pm$0.26 & 76.68$\pm$0.47 & 74.37$\pm$1.18 & 58.49$\pm$1.51 & 48.01$\pm$0.95 & 57.58$\pm$1.34 \\
SDGNN & 48.05$\pm$0.29 & 54.77$\pm$0.67 & 62.35$\pm$1.09 &  62.82$\pm$4.16 & 69.48$\pm$0.13 & 77.75$\pm$0.82 & 77.28$\pm$0.58 & 62.83$\pm$1.90 & 60.53$\pm$4.88 & 73.27$\pm$0.09  \\
SNEA & 47.61$\pm$1.26 & 49.25$\pm$0.86 & 59.30$\pm$1.32 & 57.66$\pm$0.24 & 60.35$\pm$0.45 & 79.25$\pm$0.38 & 77.36$\pm$0.27 & 62.61$\pm$0.44 & 62.21$\pm$0.16 & 70.70$\pm$0.31\\
SSSNET & 49.53$\pm$1.13 & 52.75$\pm$1.71 & 65.84$\pm$0.77 & 64.53$\pm$1.98  & 69.89$\pm$2.26 & 77.89$\pm$0.41 & 75.06$\pm$0.55 & 63.74$\pm$2.58 & 67.15$\pm$0.44  & 73.40$\pm$1.16 \\ \midrule
SigMaNet & \underline{59.59$\pm$1.68} & 60.79$\pm$0.82 & 74.09$\pm$0.14 & {78.54$\pm$0.17} & 79.12$\pm$0.22 & 81.68$\pm$0.37 & 80.92$\pm$0.36 & 74.22$\pm$0.12 & {78.31$\pm$0.06} & {82.85${\pm}$0.08} \\
MSGNN & 58.91$\pm$1.17 & \underline{63.12$\pm$0.86} & \underline{75.07$\pm$0.41} & \textbf{79.46$\pm$0.25} & \underline{80.96$\pm$0.32}  &  \underline{81.95$\pm$0.47} & \underline{82.02$\pm$0.13} & \textbf{76.63$\bm{\pm}$0.24} & \underline{78.45$\pm$0.35} & \underline{83.54$\pm$0.23} \\ \midrule
QuaterGCN & \textbf{61.74$\bm{\pm}$0.94} & \textbf{65.36$\bm{\pm}$0.84} & \textbf{75.19\bm{$\pm$}0.47} & \underline{79.21${\pm}$0.13} & \textbf{81.10$\bm{\pm}$0.18} & \textbf{82.56$\bm{\pm}$0.46} & \textbf{82.13$\bm{\pm}$0.21} & \underline{76.33$\pm$0.16} & \textbf{78.55$\bm{\pm}$0.35} & \textbf{84.03$\pm$0.09}
\end{tabular}
\end{table*}

\begin{table*}[htb!]
\setlength{\tabcolsep}{2pt}
\footnotesize
\centering
\caption{Accuracy (\%) on the five-class edge prediction tasks}
%\vspace{-.2cm}
\label{tab:five-class}
\begin{tabular}{lHHHHHrrrrr}
\hline
\multicolumn{1}{c}{} & \multicolumn{5}{H}{Four-Class Edge prediction} & \multicolumn{5}{c}{Five-Class Edge prediction}
%\multicolumn{1}{c}{} & \multicolumn{5}{c}{Four Class Edge prediction} %\\ \cline{2-6}
%\multicolumn{1}{c}{} & \multicolumn{4}{c}{Five Class Edge prediction} 
\\ \cline{2-11} 
 & \multicolumn{1}{H}{Bitcoin Alpha} & \multicolumn{1}{H}{Bitcoin OTC} & \multicolumn{1}{H}{WikiRfa} & \multicolumn{1}{H}{Slashdot} & \multicolumn{1}{H}{Epinions}   & \multicolumn{1}{c}{Bitcoin Alpha} & \multicolumn{1}{c}{Bitcoin OTC} & \multicolumn{1}{c}{WikiRfa} & \multicolumn{1}{c}{Slashdot} & \multicolumn{1}{c}{Epinions} \\ \hline
SGCN & 48.05$\pm$0.29 & 52.52$\pm$0.71 & 68.37$\pm$0.51 & 64.01$\pm$0.24 &  67.99$\pm$0.56 & 78.43$\pm$0.36 & 77.54$\pm$0.56 & 67.74$\pm$0.29 & 64.74$\pm$0.16 &  74.07$\pm$0.32\\
SiGAT & 50.12$\pm$1.80 & 50.86$\pm$1.45 & 57.68$\pm$0.63 &54.82$\pm$0.32  & 60.21$\pm$0.26 & 76.68$\pm$0.47 & 74.37$\pm$1.18 & 58.49$\pm$1.51 & 48.01$\pm$0.95 & 57.58$\pm$1.34 \\
SDGNN & 48.05$\pm$0.29 & 54.77$\pm$0.67 & 62.35$\pm$1.09 &  62.82$\pm$4.16 & 69.48$\pm$0.13 & 77.75$\pm$0.82 & 77.28$\pm$0.58 & 62.83$\pm$1.90 & 60.53$\pm$4.88 & 73.27$\pm$0.09  \\
SNEA & 47.61$\pm$1.26 & 49.25$\pm$0.86 & 59.30$\pm$1.32 & 57.66$\pm$0.24 & 60.35$\pm$0.45 & 79.25$\pm$0.38 & 77.36$\pm$0.27 & 62.61$\pm$0.44 & 62.21$\pm$0.16 & 70.70$\pm$0.31\\
SSSNET & 49.53$\pm$1.13 & 52.75$\pm$1.71 & 65.84$\pm$0.77 & 64.53$\pm$1.98  & 69.89$\pm$2.26 & 77.89$\pm$0.41 & 75.06$\pm$0.55 & 63.74$\pm$2.58 & 67.15$\pm$0.44  & 73.40$\pm$1.16 \\ \midrule
SigMaNet & \underline{59.59$\pm$1.68} & 60.79$\pm$0.82 & 74.09$\pm$0.14 & {78.54$\pm$0.17} & 79.12$\pm$0.22 & 81.68$\pm$0.37 & 80.92$\pm$0.36 & 74.22$\pm$0.12 & {78.31$\pm$0.06} & {82.85${\pm}$0.08} \\
MSGNN & 58.91$\pm$1.17 & \underline{63.12$\pm$0.86} & \underline{75.07$\pm$0.41} & \textbf{79.46$\pm$0.25} & \underline{80.96$\pm$0.32}  &  \underline{81.95$\pm$0.47} & \underline{82.02$\pm$0.13} & \textbf{76.63$\bm{\pm}$0.24} & \underline{78.45$\pm$0.35} & \underline{83.54$\pm$0.23} \\ \midrule
QuaterGCN & \textbf{61.74$\bm{\pm}$0.94} & \textbf{65.36$\bm{\pm}$0.84} & \textbf{75.19\bm{$\pm$}0.47} & \underline{79.21${\pm}$0.13} & \textbf{81.10$\bm{\pm}$0.18} & \textbf{82.56$\bm{\pm}$0.46} & \textbf{82.13$\bm{\pm}$0.21} & \underline{76.33$\pm$0.16} & \textbf{78.55$\bm{\pm}$0.35} & \textbf{84.03$\pm$0.09}
\end{tabular}
\end{table*}

%\vspace{0.3cm}
%\begin{table*}[htb!]
%\setlength{\tabcolsep}{2pt}
%\footnotesize
%\centering
%\caption{Accuracy (\%) on datasets of the five-class edge prediction}
%\label{tab:five-class}
%\begin{tabular}{lcrrrr}
%\hline
%\multicolumn{1}{c}{} & \multicolumn{5}{c}{Five Class Edge prediction} \\ \cline{2-6} 
% & \multicolumn{1}{l}{Bitcoin Alpha} & \multicolumn{1}{l}{Bitcoin OTC} & \multicolumn{1}{l}{WikiRfa} & \multicolumn{1}{l}{Slashdot} & \multicolumn{1}{l}{Epinions} \\ \hline
%SGCN & 78.43$\pm$0.36 & 77.54$\pm$0.56 & 67.74$\pm$0.29 & 64.74$\pm$0.16 &  74.07$\pm$0.32\\
%SiGAT & 76.68$\pm$0.47 & 74.37$\pm$1.18 & 58.49$\pm$1.51 & 48.01$\pm$0.95 & 57.58$\pm$1.34 \\
%SDGNN & 77.75$\pm$0.82 & 77.28$\pm$0.58 & 62.83$\pm$1.90 & 60.53$\pm$4.88 & 73.27$\pm$0.09 \\
%SNEA & 79.25$\pm$0.38 & 77.36$\pm$0.27 & 62.61$\pm$0.44 & 62.21$\pm$0.16 & 70.70$\pm$0.31 \\
%SSSNET & 77.89$\pm$0.41 & 75.06$\pm$0.55 & 63.74$\pm$2.58 & 67.15$\pm$0.44  & 73.40$\pm$1.16 \\ \midrule
%SigMaNet & 81.68$\pm$0.37 & 80.92$\pm$0.36 & 74.2
%2$\pm$0.12 & {78.31$\pm$0.06} & {82.85${\pm}$0.08} \\
%MSGNN & \underline{81.95$\pm$0.47} & \underline{82.02$\pm$0.13} & \textbf{76.63$\bm{\pm}$0.24} & \underline{78.45$\pm$0.35} & \underline{83.54$\pm$0.23} \\ \midrule
%QuaterGCN & \textbf{82.56$\bm{\pm}$0.46} & \textbf{82.13$\bm{\pm}$0.21} & \underline{76.33$\pm$0.16} & \textbf{78.55$\bm{\pm}$0.35} & \textbf{84.03$\pm$0.09}
%\end{tabular}
%\end{table*}

The task is to predict the class of each node. We consider the {\tt Telegram} dataset and the 9 aforementioned synthetic datasets, i.e., every dataset except for those that lack a pre-determined node class.

We compare \quater{} with:
(i) the three spectral GCNs designed for undirected graphs: ChebNet~\citep{defferrard2016convolutional} and GCN~\citep{kipf2016semi} and the spectral GCN with quaternionic weights QGNN designed for undirected graphs~\citep{nguyen2021quaternion}; 
(ii) the four spectral GCNs designed for directed graphs: DGCN~\citep{tong2020directed}, DiGraph~\citep{Tong2020}, DiGCL~\citep{tong2021directed}, MagNet~\citep{zhang2021magnet}, and SigMaNet~\citep{fiorini2022sigmanet}; 
and (iii) the five spatial GCNs: APPNP~\citep{klicpera2018predict}, SAGE \citep{hamilton2017inductive}, GIN~\citep{xu2018powerful}, GAT~\citep{veličković2018graph}, and SSSNET~\citep{he2022sssnet}.
The experiments are run with $10$-fold cross-validation with a 60\%/20\%/20\% split for training, validation, and testing.

Table~\ref{tab:node} and~\ref{tab:node1} show that \quater{} achieves a remarkable performance across all datasets, being the best method in 9 cases out of 10. The percentage difference between \quater{} and the second-best performer ranges from $0.2\%$ (for {\tt DBSM} with $\alpha_{uv}=0.1$) to $242.81\%$ (for {\tt Di500} with $\delta=0.7$). The average performance improvement of \quater{} across all datasets is $68.27\%$.
%
%\red{{\bf Talking about SigManet.} Notably, for synthetic graphs with $\alpha_{uv} = {0.1, 0.08, 0.05}$, which do not feature digons, our Laplacian behaves exactly like the sign-Magnetic Laplacian. Therefore, the superior performance obtained by \quater{} can be attributed exclusively to the quaternion neural network.}
%\iln{MC: questo paragrafo la toglierei. Le differenze sono marginali, meglio non far risaltare troppo la cosa. }
%
%This confirms that the encoding of the graph topology achieved by our proposed Laplacian plays a crucial role in successfully accomplishing this task.
\quater{} consistently outperforms the state of the art on, in particular, {\tt Di500} and {\tt Di150}, where it achieves an average improvement w.r.t. the second-best performer of 28.19\% on the {\tt Di150} dataset and of 175\% on the {\tt Di500} one.
The larger improvement of \quater{} and, in particular, the overall weaker performance of every other method on the {\tt Di500} dataset seems to be correlated with the dataset being sparser than the smaller {\tt Di150}, which suggests that the learning task is harder.
% due to it being sparser than {\tt Di150} to the task being harder on sparse graphs. Indeed, denser graphs such as the {\tt Di150} ones feature more digons, \red{making the task more difficult for approaches that disregard them.} \sc{non si capisce}
%which enhances the characteristics of our approach and results in a significant performance boost. In contrast, other models exhibit only marginal improvements under similar conditions.
% \sc{questa sezione in blu va rivista, sembra proporre tesi opposte}
%
The difference between \quater{} and QGNN (the only available GCN with quaternion-valued weights which, though, employs the classical real-valued Laplacian), is substantial, as \quater{} outperforms QGNN by 309\% on average. As the two networks share a similar architecture, we attribute such a large difference to \quater{}'s convolution being done via our proposed Quaternionic Laplacian $L^\ku$.

%\st{The consistently remarkable performance (best method on 9 to 10 datasets) of \quater{} clearly highlights its effectiveness and the effectiveness of the proposed Laplacian when adopted within a spectral CGN.}

%\red{far notare che performance del 20\% è la perfromance del predittore banale che ritorna sempre la stessa classE}

\subsection{Three-Class Edge Prediction Task (3CEP)}\label{sub:threeclass}

%\red{SC: dire che: non usiamo le istanze "digon=x" perchè su di loro tutte le reti (non solo QuaNet) hanno performance di circa il 30\%: le task non sono interessanti}
%
The task is to predict whether $(u,v) \in E$, $(v,u) \in E$, or $(u, v) \notin E \wedge (v, u) \notin  E$.
In order to maximize the number of spectral methods we can compare to, for this task our analysis focuses on the datasets with positive weights, i.e., {\tt Telegram}, {\tt DSBM}, {\tt Di150}, and {\tt D500}. Following \citet{fiorini2022sigmanet}, we also consider {\tt Bitcoin Alpha$^*$} and {\tt Bitcoin OTC$^*$}, 
%which are 
obtained by removing any negative-weight edge from {\tt Bitcoin Alpha} and {\tt Bitcoin OTC}.

%Wikirfa non lo usiamo perche' siamo scarsi, ma non e' un problema perche' ha pesi anche negativi e non cade nella nostro argomento di scelta.

We compare \quater{} to the same methods we considered for the NC task.
%SSSNET dovrebbe esserci ma manca: facciamo gli esperimenti mentre il papero e' in revisione nel caso ce li chiedano (e aggiungiamoli in caso di acceptance?)
%
We run the experiments with $10$-fold cross-validation with an 80\%/15\%/5\% split for training, testing, and validation, preserving graph connectivity.

%The datasets denoted by a ‘*’ are pre-processed as in~\citet{fiorini2022sigmanet}.
%
%We have decided not to evaluate the instances with digons and $N = 500$ based on a preliminary experimental analysis. The performance across different networks (not just \quater{}) on this classification task was consistently around 30\%. Therefore, the results were considered uninteresting.
%

The results are reported in Table~\ref{tab:threeclass}. Those obtained on the {\tt Di500} dataset are omitted as on it every method achieves the same performance of about 30\% (equal to a uniform random predictor). The table shows that \quater{} outperforms the other methods on 7 datasets out of 9. Compared with the second-best model, \quater{} achieves an average performance improvement of 0.98\%, with a maximum of 3.60\% and a minimum of 0.02\%.

Differently from the NC task, in the 3CEP task the difference in performance between the methods is smaller, as already observed by~\citet{zhang2021magnet} and \citet{ fiorini2022sigmanet}. Nevertheless, the results indicate that the advantages provided by the \laplaciano{} are still event, albeit being of smaller magnitude.

Focusing on QGNN, 
%(which
%, we recall, 
%features quaternionic weights and a real-valued Laplacian), 
\quater{} outperforms it on 9 out of 10 datasets by an average of 25.58\%.
This further reinforces that the better performance of \quater{} is largely due to it relying on our proposed Laplacian matrix $L^\ku$ rather than on the mere adoption of quaternionic weights as done in QGNN with the classical Laplacian.
%
%Comparing \quater{} to a network employing quaternion-valued weights but a real-valued Laplacian, we can assess the benefits of \quater{} that are only due to its quaternion-valued Laplacian. Specifically, \quater{} outperforms QGNN in 9 out of 10 datasets by an average of 25.58\%.}
%
%\red{The results obtained in the three-class edge prediction task demonstrate the superior performance of \quater{} compared to other models. However, in contrast to node classification tasks, the observed performance improvements are relatively more modest. Nevertheless, the utilization of \quater{} with a quaternion neural network still provides certain advantages, albeit to a lesser extent.}
%

Table~\ref{tab:threeclass} suggests that simpler methods designed for undirected graphs perform increasingly better when $\delta$ increases on the {\tt Di150} dataset.
This is likely due the fact that these graphs feature a small difference in edge weight. If $A_{uv} \simeq A_{vu}$, not much is lost if the (almost symmetric) digon $(u,v),(v,u)$ is reduced to a single undirected edge of weight $\frac{1}{2}(A_{uv}+A_{vu})$, as done when, e.g., using the classical real-valued Laplacian matrix, as this would only lead to a small loss of information, if any.
What is more, the larger the number of digons, the more the graph becomes close to being undirected if the difference in weight is small, which explains why simpler (and arguably easier to train) methods designed for undirected graphs achieve an increasingly better performance as the percentage of digons $\delta$ increases.
%SI VEDE CHE I GRAFI DOVE LA DIFF. IN PERFORMANCE È MAGGIORE HANNO DIFF DI PESO GRANDE E/O DENSITÀ PIÙ BASSA?

\subsection{Four/Five-Class Edge Prediction Task (4/5CEP)}

The 4CEP task is to predict whether $(u, v) \in E^+$, $(u, v) \in E^-$, $(v, u) \in E^+$, and $(v, u) \in E^-$ (with $E^+$ and $E^-$ being the positive- and negative-weight edges), while the 5CEP task also considers the class where $(u, v) \notin E \wedge (v, u) \notin  E$. %in which the model is asked to predict whether $(u, v) \in E^+$, $(u, v) \in E^-$, $(v, u) \in E^+$, and $(v, u) \in E^-$.  %and {\em five-class edge prediction}. In the first one, the model is asked to predict whether $(u, v) \in E^+$, $(u, v) \in E^-$, $(v, u) \in E^+$, and $(v, u) \in E^-$. In the second one, we also add the possibility that neither $(u, v) \notin E \wedge (v, u) \notin  E$. 
Due to their nature, for both tasks we focus on every dataset featuring both positive and negative weights, i.e., on all the real-world datasets except for {\tt Telegram}.
%We take into account five different datasets: {\tt Bitcoin Alpha}, {\tt Bitcoin-OTC}, {\tt WikiRfa}, {\tt Slashdot}, and {\tt Epinions}.} %, which are directed graphs with weights of unrestricted sign (necessary for the task to be applicable) and of arbitrary magnitude, with the sole exception of the last two, whose weights satisfy $A \in \{-1,0, +1\}^{n \times n}$. In these five datasets, the classes of positive and negative weighted edges are imbalanced (i.e., nearly 80\% are positive edges).
%
We run the experiments with $5$-fold cross-validation with an 80\%/20\% split for training and testing, preserving graph connectivity.
%Following~\cite{fiorini2022sigmanet}, we reserve 20\% of the edges for testing and use the remaining 80\% for training. The experiments are run with $k$-cross validation with $k=5$, preserving graph connectivity.

Due to the nature of the tasks, we compare \quater{} against the only methods that can handle the sign of the edge weights, i.e.: \textit{i)} the signed graph neural networks SGCN~\citep{derr2018signed}, SiGAT~\citep{huang2019signed}, SNEA~\cite{li2020learning}, SDGNN~\citep{huang2021sdgnn}, and SSSNET~\citep{he2022sssnet}; and \textit{ii)} the spectral GCNs that are well-defined for negative edge weights, i.e., SigMaNet~\citep{fiorini2022sigmanet} and MSGNN~\cite{he2022msgnn}.

Table~\ref{tab:four-class} and~\ref{tab:five-class} show that, when compared to the other approaches, \quater{} achieves superior performance in 8 out of 10 cases, while being the second-best model in the other 2. In comparison with the second-best model, \quater{} achieves an average performance improvement of 1.14\%, with a maximum of 3.61\% and a minimum of 0.13\%.

While not as large as for the NC task, the better performance that \quater{} achieves on the 4CEP and 5CEP tasks confirms the superior performance of the model we proposed.
%\red{These findings provide evidence for the versatility of our model across various scenarios and tasks, allowing it to outperform state-of-the-art methods.}

%SF: Aggiungere più commenti riguardo al nuemro di archi antipararelli, tabella in appendice.
%OMESSI PERCHE' NON UNIFORMI SUI DUE DATASETS SINTETICI

\section{Conclusions}\label{sec:conclusion}

We have proposed the \textit{\laplaciano{}} $L^\ku$, a quaternion-valued graph Laplacian which generalizes different previously-proposed Laplacian matrices while allowing for the seamless representation of graphs and digraphs of any weight and sign featuring any number of digons without
%(differently from other Laplacians)
suffering from losses of topological information.
We have then proposed \quater{}, a spectral GCN with quaternionic network weights that employs a quaternion-valued convolution operator built on top of $L^\ku$.
Our extensive experimental campaign has highlighted the advantages of employing our quaternion-valued graph Laplacian matrix to leverage the full topology of input graphs featuring digons.
Future works include extending the \laplaciano{} to multi-relational graphs with multiple directional edges and to temporal (time-extended) graphs.

\section* {Acknowledgements}
This work was partially funded by
the PRIN 2020 project \textit{ULTRA OPTYMAL - Urban Logistics and sustainable TRAnsportation: OPtimization under uncertainty and MAchine Learning} (grant number 20207C8T9M)
and the PRIN-PNRR 2022 project {\em HEXAGON: Highly-specialized EXact Algorithms for Grid Operations at the National level} (grant number P20227CYT3),
both funded by the Italian University and Research Ministry.

\clearpage
% Use \bibliography{yourbibfile} instead or the References section will not appear in your paper
%\nobibliography{aaai23}
%\bibliographystyle{plainnat}
\bibliography{BIB}

 \clearpage
% %\maketitle
 \appendix
 \everymath{\allowdisplaybreaks}

\section{Code Repository and Licensing}\label{appx:implementation}

The code written for this research work (which, in particular, implements \quater{}) is available at \url{https://anonymous.4open.science/r/QuaterGCN} and freely distributed under the Apache 2.0 license.\footnote{\url{https://www.apache.org/licenses/LICENSE-2.0}}

%As for asset licensing, 
The datasets (except for WikiRfa) and some code components were obtained from the PyTorch Geometric Signed Directed~\citep{he2022pytorch} library (provided under the MIT license).
The WikiRfa dataset is available at \url{https://networks.skewed.de/net/wiki_rfa} under the BSD licence.\footnote{\url{https://choosealicense.com/licenses/bsd-2-clause/}}
%x
The methods used for the experimental analysis are available at \url{https://github.com/stefa1994/sigmanet} under the MIT license.\footnote{\url{https://choosealicense.com/licenses/mit/}}

\section{Properties of the \laplaciano{}}\label{appx:theorem}

This section contains the proofs of the theorems reported in the main paper.
%Section~\ref{sec:laplaciano}. 

\setcounter{theorem}{0}
\begin{theorem}
$L^{\ku} = L$ for every graph $G$ with $A$ symmetric and nonnegative and $D_{vv} > 0$ for all $v \in V$.
\end{theorem}
\begin{proof}
As $G$ is undirected, its adjacency matrix $A$ is symmetric. Since, for such any such graph, $W^D = 0$ and $N^D = 0$ hold, we deduce $H^{\ku} = H^0$. This implies $\symbolapl = L$, from which the claim follows.
\end{proof}

\begin{theorem}%\label{thm:sigmanet}
$L^{\ku} = L^{\sigma}$ if, for all $u,v \in V$, either $A_{uv} = 0$, or $A_{uv} = A_{vu}$.
\end{theorem}
\begin{proof}
%Since $G$ is a weighted directed graph, $D^{\sigma}$ and $\bar D_s$ coincide.
%
Given a weighted directed graph, for each $u,v \in V$, if $A_{uv} \neq 0$ and $A_{vu} = 0$, we have $H^{\sigma}_{uv} = - H^{\sigma}_{uv} = 0 + \ii \frac{1}{2} = H^{\ku}_{uv} = - H^{\ku}_{vu}= 0 +  \ii \frac{1}{2} + \jj0 + \kk0 $; if $A_{uv} = A_{vu} = a$, with $a$ being some nonnegative real value, we have $H^{\sigma}_{uv} = H^{\sigma}_{uv} = 1 + \ii 0 = H^{\ku}_{uv} = H^{\ku}_{vu}= 1 +  \ii0 + \jj0 + \kk0$.
Thus, since $H^{\ku} = H^\sigma$, we have $L^{\ku} = L^{\sigma}$ and the claim follows.
\end{proof}
\begin{theorem}
    Consider a weighted digraph without symmetric digons and let $H^m := \left( A_s^2 \odot H^2_1 + A_s^3 \odot H^3_1 \right)$, where $H^2_1 = N \odot \big(U(T) - L(T^\top) \big)$ and $H^3_1 = - H^2_1$. We have $H^0 = 0$ and $H^{\ku} = H^{\sigma} \odot \left( e e^\top - O \right) + O \odot H^m$. Thus, each component of $H^{\ku}_{uv}$ is a linear combination of the component $H^\sigma_{uv}$ of the Sign-Magnetic Laplacian $H^{\sigma}$ and of the components $H^m_{uv}$ of the quaternionic Hermitian matrix $H^m$.
\end{theorem}
In the statement, the matrix $H^m$ coincides, by construction, with the imaginary parts $\Im_2(H^\ku)$ and $\Im_3(H^\ku)$ of the matrix $H^\ku$ and, therefore, encodes all the digons of the graph that do not have equal weights.
\begin{proof}
%Notice that $\left( ee^\top  - \sgn (|A - A^\top| \right) = \left( ee^\top  - \sgn (|A - A^\top| \right) \odot \big( e e^\top - \sgn(|A|) \odot \sgn(|A|^\top)\big)$. Therefore, the real component of $H^{\ku}$ can be w.l.o.g. rewritten as $A_s \odot \left( ee^\top  - \sgn (|A - A^\top| \right) \odot \big( e e^\top - \sgn(|A|) \odot \sgn(|A|^\top)\big)$.
    %
    %Collecting the first two terms of $H^{\ku}$ by $\big( e e^\top - \sgn(|A|) \odot \sgn(|A|^\top)\big)$ reveals them to coincide with $H^\sigma \odot \big( e e^\top - \sgn(|A|) \odot \sgn(|A|^\top)\big)$. Collecting the second two terms by $\sgn(|A|) \odot \sgn(|A|^\top)$ revelas them to coincide with 
    %$H^{m} \odot \big( \sgn(|A|) \odot \sgn(|A|)^\top \big)$.

    Given a weighted digraph without any digons of equal weight, $H^{\ku} = A^1_s \odot H^0 + \ii A^1_s \odot H^1 + \jj A_s^2 \odot  H^{2} + \kk A_s^3 \odot H^{3}$, with $H^0 = 0$ due to the nature of the graph under consideration.
    Collecting the first term of $H^{\ku}$ by $\big( e e^\top - O^D \big)$ shows that it coincides with $H^\sigma \odot \big( e e^\top - O^D \big)$. 
    Collecting the second two terms by $O^D$ shows that they coincide with $\big( H^{m} \odot O^D \big)$.
    This concludes the proof.
\end{proof}

\newpage
\begin{theorem}
$L^{\ku}$ and $L^{\ku}_{norm}$ are positive semidefinite.
\end{theorem}
\begin{proof}
By  definition of $L^\ku$, we have $\Re(\symbolapl) = \bar D_s - A^1_s \odot H^0$, $\Im_1(\symbolapl) = - A^1_s \odot H^1$, $\Im_2(\symbolapl) = - A^2_s  \odot H^2$ and $\Im_3(\symbolapl) = - A^3_s  \odot H^3$. By construction, $\Re(\symbolapl)$ is symmetric and $\Im_1(\symbolapl)$, $\Im_2(\symbolapl)$, and $\Im_3(\symbolapl)$ are skew symmetric. It follows that $\symbolapl$ is a Hermitian matrix, which implies that $x^*\Im_1(\symbolapl)x=0$, $x^*\Im_2(\symbolapl)x=0$ and $x^*\Im_3(\symbolapl)x=0$ hold for all $x \in \HH^{n}$.
As, by construction, $\bar D_s = \Diag(|A^1_s| \, e)$ and $A^1_s$ is symmetric, we can show that $x^*\Re(\symbolapl)x \geq 0$ holds for all $x \in \HH^{n}$ via the following derivation:

\scriptsize
$2x^*\Re(\symbolapl)x$
\begin{align*}
    & = 2 \sum_{u,v=1}^n (\bar D_s)_{uv} x_u x_v^* - 2 \sum_{u,v=1}^n (A^1_s)_{uv}x_u x_v^*  \left(1 - H^0_{uv}\right)\\
    & = 2 \sum_{u,v=1}^n (\bar D_s)_{uv} x_u x_v^* - 2 \sum_{u,v=1}^n (A^1_s)_{uv}x_u x_v^*  \left(1 - \sgn(|A_{uv} - A_{vu}|)\right)\\
    & = 2 \sum_{i=1}^n (\bar D_s)_{uu} x_u x_u^* - 2 \sum_{u,v=1}^n (A^1_s)_{uv}x_u x_v^* \left(1 - \sgn(|A_{uv} - A_{vu}|)\right) \\
    & = 2 \sum_{u,v=1}^n |(A^1_s)_{uv}| |x_u|^2 - 2 \sum_{u,v=1}^n (A^1_s)_{uv}x_u x_v^* \left(1 - \sgn(|A_{uv} - A_{vu}|)\right)\\
    & = \sum_{u,v=1}^n |(A^1_s)_{uv}| |x_u|^2 + \sum_{u,v=1}^n |(A^1_s)_{vu}| |x_v|^2 \\ 
    & - 2 \sum_{u,v=1}^n (A^1_s)_{uv}x_u x_v^* \left(1 - \sgn(|A_{uv} - A_{vu}|)\right)\\
    & = \sum_{u,v=1}^n |(A^1_s)_{uv}| |x_u|^2 + \sum_{u,v=1}^n |(A^1_s)_{uv}| |x_v|^2\\ & - 2 \sum_{u,v=1}^n (A^1_s)_{uv}x_u x_v^* \left(1 - \sgn(|A_{uv} - A_{vu}|)\right)\\
    & = \sum_{u,v=1}^n |(A^1_s)_{uv}| |x_u|^2 + \sum_{u,v=1}^n |(A^1_s)_{uv}| |x_v|^2 \\ & - 2 \sum_{u,v=1}^n |(A^1_s)_{uv}| \sgn((A^1_s)_{uv}) x_u x_v^* \left(1 - \sgn(|A_{uv} - A_{vu}|)\right)\\
    & = \sum_{u,v=1}^n |(A^1_s)_{uv}| \Big(|x_u|^2 + |x_v|^2 - 2 \sgn(A_{uv}) x_u x_v^* \left(1 - \sgn(|A_{uv} - A_{vu}|)\right)\Big) \;\\
    & \geq \sum_{u,v=1}^n |(A^1_s)_{uv}| \Big(|x_u| - \sgn(A_{uv})|x_v| \Big)^2 \geq 0.
\end{align*}
\normalsize
This shows, as desired, that $\symbolapl \succeq 0$.
Let us now consider the {\em normalized \laplaciano{}}, which, according to Eq.~\eqref{eq:norma}, is defined as $\symbolapl_\text{norm} = \rsqrt{\bar D_s} \symbolapl \rsqrt{\bar D_s}$. We need to show that $x^*\symbolapl_{\text{norm}}x \geq 0$ for all $x \in \HH^n$. Letting $y =\rsqrt{\bar D_s}x$, we have $x^*\symbolapl_{\text{norm}}x = x^*\rsqrt{\bar D_s} \symbolapl \rsqrt{\bar D_s}x = y^*\symbolapl y$. As the latter quantity was proven before to be nonnegative for all $y \in \HH^n$, $\symbolapl_{\text{norm}} \succeq 0$ follows.
\end{proof}

\begin{theorem} 
$\lambda_\text{max}(\symbolapl_{\text{norm}}) \leq 2$.
\end{theorem}
\begin{proof}

As the Courant-Fischer theorem applied to $\symbolapl_{\text{norm}}$ implies:
\begin{equation*}
 \lambda_{\text{max}}= \max_{x \neq 0} \frac{x^*\symbolapl_{\text{norm}}x}{x^*x},  
\end{equation*}
we need to show that the following holds:

\begin{equation*}
    \max_{x \neq 0} \frac{x^*\symbolapl_{\text{norm}}x}{x^*x} \leq 2.
\end{equation*}

For the purpose, let $B := \bar D_s + H^{\ku}$ and let $B_{\text{norm}} := \rsqrt{\bar D_s} B \rsqrt{\bar D_s}$ be its normalized counterpart. Assuming $B_{\text{norm}} \succeq 0$, we would deduce the following for all $x \in \HH^n$:
\begin{align*}
    & x^*B_{\text{norm}}x \geq 0 \Leftrightarrow  x^*\left( I + \rsqrt{D}H^{\ku}\rsqrt{D} \right)x \geq 0 \Leftrightarrow \\
    & -x^* \rsqrt{D} H^{\ku} \rsqrt{D}x \leq x^*x \Leftrightarrow \\
    & x^*Ix - x^* \rsqrt{D}H^{\ku} \rsqrt{D}x \leq 2x^*x.
\end{align*}
This would suffice to prove the claim as $x^*Ix - x^* \rsqrt{D}H^{\ku} \rsqrt{D}x \leq 2x^*x$ implies that $\frac{x^* \symbolapl_{\text{norm}}x}{x^*x} \leq 2$ holds for all $x \in \HH^n$.

Letting $y = \rsqrt{\bar D_s}x$, we observe that  $x^*B_{\text{norm}}x = x^*\rsqrt{\bar D_s} B \rsqrt{\bar D_s}x \geq 0$ holds for all $x \in \HH^n$ if and only if $y^*By\geq  0$ holds for all $y \in \HH^n$. Therefore, it suffices to show that $B \succeq 0$.

As $B$ is Hermitian by construction, we have $x^*\Im_1(B)x = x^*\Im_2(B)x = x^*\Im_3(B)x= 0$ for all $x \in \HH^n$. Thus, we only need to show that $x^*\Re(B)x \geq 0$ holds for all $x \in \HH^n$. This is done via the following derivation:

\scriptsize
$2x^*\Re(B)x$
\begin{align*}
& = 2 \sum_{u,v=1}^n (\bar D_s)_{uv} x_u x_v^* + 2 \sum_{u,v=1}^n (A^1_s)_{uv}x_u x_v^*  \left(1 - H^0_{uv}\right)\\
    & = 2 \sum_{u,v=1}^n (\bar D_s)_{uv} x_u x_v^* + 2 \sum_{u,v=1}^n (A^1_s)_{uv}x_u x_v^*  \left(1 - \sgn(|A_{uv} - A_{vu}|)\right)\\
    & = 2 \sum_{i=1}^n (\bar D_s)_{uu} x_u x_u^* + 2 \sum_{u,v=1}^n (A^1_s)_{uv}x_u x_v^* \left(1 - \sgn(|A_{uv} - A_{vu}|)\right) \\
    & = 2 \sum_{u,v=1}^n |(A^1_s)_{uv}| |x_u|^2 + 2 \sum_{u,v=1}^n (A^1_s)_{uv}x_u x_v^* \left(1 - \sgn(|A_{uv} - A_{vu}|)\right)\\
    & = \sum_{u,v=1}^n |(A^1_s)_{uv}| |x_u|^2 + \sum_{u,v=1}^n |(A^1_s)_{vu}| |x_v|^2\\
    & + 2 \sum_{u,v=1}^n (A^1_s)_{uv}x_u x_v^* \left(1 - \sgn(|A_{uv} - A_{vu}|)\right)\\
    & = \sum_{u,v=1}^n |(A^1_s)_{uv}| |x_u|^2 + \sum_{u,v=1}^n |(A^1_s)_{uv}| |x_v|^2\\ 
    & + 2 \sum_{u,v=1}^n (A^1_s)_{uv}x_u x_v^* \left(1 - \sgn(|A_{uv} - A_{vu}|)\right)\\
    & = \sum_{u,v=1}^n |(A^1_s)_{uv}| |x_u|^2 + \sum_{u,v=1}^n |(A^1_s)_{uv}| |x_v|^2 \\
    & + 2 \sum_{u,v=1}^n |(A^1_s)_{uv}| \sgn((A^1_s)_{uv}) x_u x_v^* \left(1 - \sgn(|A_{uv} - A_{vu}|)\right)\\
    & = \sum_{u,v=1}^n |(A^1_s)_{uv}| \Big(|x_u|^2 + |x_v|^2 + 2 \sgn(A_{uv}) x_u x_v^* \left(1 - \sgn(|A_{uv} - A_{vu}|)\right)\Big) \;\\
    & \geq \sum_{u,v=1}^n |(A^1_s)_{uv}| \Big(|x_u|^2 + |x_v|^2 \Big) \geq 0.
\end{align*}
\normalsize
This concludes the proof.
\end{proof}

\section{QuaterGCN's architecture}

\begin{figure*}[ht]
%\vspace{-.3cm}
\includegraphics[width=\textwidth]{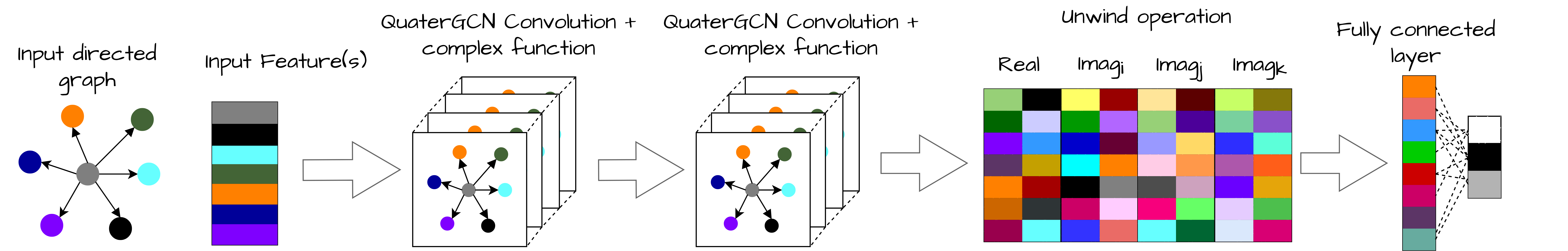}
%\vspace{-0.5cm}
\caption{Overview of \quater{}}
%\vspace{-0.3cm}
\label{fig:architecture}
\end{figure*}

Based on the task to be solved, \quater{} features one or more convolution layers followed by either a linear layer with weights $W$ or a $1D$ convolution. 
Considering, for example, a node-classification task of predicting which of a set of unknown classes a graph vertex belongs to, \quater{} implements the function
\small
\begin{equation*}
    \text{unwind}\left(Z^{\ku(2)} \left(Z^{\ku(1)} \left(X^{(0)} \right)\right)\right)       W,
\end{equation*}
\normalsize
where $X^{(0)} \in \HH ^{n \times c}$ is the input feature matrix, $Z^{\ku(1)} \in \HH^{n \times f_1}$ and $Z^{\ku(2)} \in \HH^{n \times f_2}$ are two convolutional layers, and $W \in \RR^{4f_2 \times d}$ are the weights of the linear layer (with $d$ being the number of classes).

Fig.~\ref{fig:architecture} illustrates the structure of the proposed \quater{} network for the node-classification task.  

\section{Further Details on the Datasets}\label{appx:dataset}

\paragraph{Real-World Datasets.} We tested \quater{} on six real-world datasets: {\tt Bitcoin-OTC} and {\tt Bitcoin Alpha}~\citep{7837846}; {\tt Slashdot} and {\tt Epinions}~\cite{leskovec2010signed}; {\tt WikiRfa}~\citep{west2014exploiting}; and {\tt Telegram}~\citep{bovet2020activity}. 

{\tt Bitcoin-OTC} and {\tt Bitcoin Alpha} are built from exchange operations executed within the Bitcoin-OTC and Bitcoin Alpha platforms, which were rated by the users with values between $-10$ to $+10$ ($0$ is not allowed). Scammers are given a score of $-10$. A score of $+10$ indicates full trust. Values in between indicate intermediate evaluations.

The other two datasets are {\tt Slashdot} and {\tt Epinions}. The first comes from a tech news website with a community of users. The website introduced Slashdot Zoo features that allow users to tag each other as friend or foe. The dataset represents a signed social network with friend $(+1)$ and enemy $(-1)$ labels. {\tt Epinions} is an online who-trust-who social network of a consumer review site (\textit{Epinions.com}). Site members can indicate their trust or distrust of other people's reviews. The network reflects people's views on others.
%Therefore, these two exchanges explicitly lead to a graph with weights unrestricted in sign. 

{\tt WikiRfa} is a collection of votes given by Wikipedia members gathered from 2003 to 2013. 
Indeed, any Wikipedia member can vote for support, neutrality, or opposition to a Wikipedia editor's nomination for administrator. 
This leads to a directed, %weighted 
%\iln{Ho tolto weighted perchè i valori sono 0 e 1}
multigraph (unrestricted in sign) in which nodes represent Wikipedia members and edges represent votes, which is then transformed into a simple graph by condensing any parallel edges into a single edge of weight equal to the sum of the weights of the original edges. The graph features a higher number of nodes and edges than the one proposed in~\cite{huang2021sdgnn}. %\iln{SC: e' un multigrafo con un lato per voto? lo trasformiamo in un grafo semplice allora?}
In these five datasets, the classes of positive and negative edges are imbalanced (see Table \ref{tab:stats}). 

{\tt Telegram} models an influence network built on top of interactions among distinct users who propagate ideologies of a political nature. The graph encompasses 245 Telegram channels with 8912 links, with labels of four classes generated as explained in~\cite{bovet2020activity}.

%\begin{table*}[htb]
%\small
%\centering
%\caption{Statistics of the three real-world datasets}
%\label{tab:stats}
%\begin{tabular}{lrrrrccr}
%\hline
%Data set & $n$ & $|\varepsilon^+|$ & $|\varepsilon^-|$ & \% pos & Directed & Weighted & Density \\ \hline
%{\tt Telegram} & 245 & 8,912 & 0 & 100.00 & \checkmark & \checkmark & 14.91\% \\
%{\tt Bitcoin-Alpha} & 3,783 & 22,650 & 1,536 & 93.65 & \checkmark & \checkmark & 0.17\% \\
%{\tt Bitcoin-OTC} & 5,881 & 32,029 & 3,563 & 89.99 & \checkmark & \checkmark & 0.10\% \\
%%{\tt DBSM with $\alpha = 0.05$} & 2,500 & 187,437 & 0 & 100.00 & \checkmark & \checkmark & 3.00\% \\
%%{\tt DBSM with $\alpha = 0.08$} & 2,500 & 262,542 & 0 & 100.00 & \checkmark & \checkmark & 4.20\% \\
%%{\tt DBSM with $\alpha = 0.10$} & 2,500 & 312,992 & 0 & 100.00 & \checkmark & \checkmark & 5.00\% \\ 
%\hline
%\end{tabular}
%\end{table*}

\begin{table*}[htb]
\setlength{\tabcolsep}{4pt}
\centering
\caption{Statistics of the six datasets}
\label{tab:stats}
\begin{tabular}{lrrrrcccrr}
\hline
Data set & $n$ & $|\varepsilon^+|$ & $|\varepsilon^-|$ & \% pos & Directed & Weighted & Density & \multicolumn{1}{c}{\begin{tabular}[c]{@{}c@{}}Undirected\\ Edges\end{tabular}} & \multicolumn{1}{c}{\begin{tabular}[c]{@{}c@{}}Antiparallel Edges\\ with different/\\
opposite weight\end{tabular}} \\ \hline
{\tt Telegram} & 245 & 8,912 & 0 & 100.00 & \checkmark & \checkmark & 14.91\% & 2.22\% & 15.42\% \\
{\tt Bitcoin-Alpha} & 3,783 & 22,650 & 1,536 & 93.65 & \checkmark & \checkmark & 0.17\% & 59.57\% & 23.63\% \\
{\tt Bitcoin-OTC} & 5,881 & 32,029 & 3,563 & 89.99 & \checkmark & \checkmark & 0.10\% & 56.89\% & 22.34\% \\
{\tt WikiRfA} & 11,381 & 138,143 & 39,038 & 77.97 & \checkmark & \checkmark & 0.14\% & 5.27\% & 2.14\% \\
{\tt Slashdot} & 82,140 & 425,072 & 124,130 & 77.70 & \checkmark & \xmark & 0.01\% & 17.03\% & 0.71\% \\
{\tt Epinion} & 131,828 & 717,667 & 123,705 & 85.30 & \checkmark & \xmark & 0.01\% & 30.16\% & 0.64\% \\ \hline
\end{tabular}
\end{table*}

%\iln{Dire che wikirfa è diverso nell'altro esperimento ma qui siamo nelal condizione peggiore}

%\paragraph{Synthetic datasets.} Abbiamo utilizzato tre dataset sintetici proposti da~\citet{fiorini2022sigmanet}: {\tt DBSM with $\alpha = 0.05$}, {\tt DBSM with $\alpha = 0.08$}, and {\tt DBSM with $\alpha = 0.10$}. They were generated via a direct stochastic block model (DSBM) with edge weights in the range $\NN \cap [2, 1000]$. 

\paragraph{Synthetic Datasets.} The synthetic graphs we adopted belong to two classes:
\begin{enumerate}
    \item {\tt DBSM}: graphs obtained via a direct stochastic block model (DSBM) following~\citet{fiorini2022sigmanet}, with edge weights taking integer values in $\{2, 1000\}$.
    \item {\tt Di150} and {\tt Di500}: since the {\tt DBSM} graphs do not contain any digons, we introduce a second class of synthetic graphs with a variable percentage of digons $\delta \in (0, 1)$ with edge weights taking integer values in $\{2, 4\}$. {\tt Di150} contains graphs with 150 nodes per cluster, whereas {\tt Di500} contains graphs with 500 nodes per cluster.
\end{enumerate}

For the {\tt DSBM} graphs, we select a number of nodes in each cluster $N$ and a number of clusters $C$ by which the vertices are partitioned into groups of equal size. We introduce a set of probabilities $\{\alpha_{uv}\}_{1 \leq u,v \leq C}$, where $0 \leq \alpha_{uv} \leq 1$ with $\alpha_{uv} = \alpha_{vu}$. Such values coincide with the probability of generating an undirected edge between nodes $u$ and $v$ taken from different clusters.
%, i.e., $u \in C_u$ and $v \in C_v$,.
$\alpha_{uu}$ is the probability of generating an undirected edge between nodes in the same cluster. As these graphs are undirected, we follow~\citet{fiorini2022sigmanet} and introduce a rule to transform the graph from undirected to directed: we define a collection of probabilities $\{\beta_{uv}\}_{1 \leq u,v \leq C}$ with $0 \leq \beta_{uv} \leq 1$ such that $\beta_{uv} + \beta_{vu} = 1$. Each edge $\{u, v\}$ is assigned a direction using the rule that the edge goes from $u$ to $v$ with probability $\beta_{uv}$ %if $u \in C_u$ and $v \in C_v$ 
and from $v$ to $u$ with probability $\beta_{vu}$.

For the creation of the {\tt Di150} and {\tt Di500} graphs, we introduce an additional set of probabilities $\delta$ with $0 < \delta < 1$ that specify what percentage of edges from those generated by the aforementioned procedure should be preserved as undirected; all the other edges are transformed into a digon.

\section{Experiment Details}\label{appx:experiment}

\paragraph{Hardware.} The experiments were conducted on 3 different machines: the first features 1 NVIDIA Tesla T4 GPU, 380 GB RAM, and Intel(R) Xeon(R) Gold 6238R CPU @ 2.20GHz CPU; the second features 1 NVIDIA RTX 3090 GPU, 64 GB RAM, and 12th Gen Intel(R) Core(TM) i9-12900KF CPU @ 3.20GHz CPU; the third features 1 NVIDIA Ampere A100, 384 GB RAM, and 2x Intel(R) Xeon(R) Silver 4210 CPU @ 2.20GHz Sky Lake CPU.

\paragraph{Model Settings.} We train every learning model considered in this paper for up to 3000 epochs with early stopping whenever the validation error does not decrease after 500 epochs for node classification and three-class edge prediction tasks.
For the four-class edge prediction task and the five-class edge prediction task, we set the number of epochs to 300 for {\tt Bitcoin\_Alpha}, {\tt Bitcoin\_OTC} and {\tt WikiRfa}, while 500 epochs for {\tt Slashdot} and {\tt Epinions}.
%\iln{aggiugere che i primi tre daatset training fino a 300 epoche, mentre slashdot ed epinions 500 perché più grandi.}
%
Following~\citet{fiorini2022sigmanet}, we add a single dropout layer with a probability of $0.5$ before the last layer. We set the parameter $K$ adopted in ChebNet, MagNet, SigMaNet, QGNN, and \quater{} to 1. 
We adopt a learning rate of $10^{-3}$ for the node classification task and the three-class edge prediction task, whereas for the other two tasks (4/5CEP), we use a learning rate of $10^{-2}$. We employ the optimization algorithm ADAM with weight decays equal to $5 \cdot 10^{-4}$ (in order to avoid overfitting).
For SGCN, SNEA, SiGAT, and SDGNN, we use the same setting proposed in~\citet{he2022msgnn}.

A hyperparameter optimization procedure is adopted to identify the best set of parameters for each model. In particular, we tune the number of filters by selecting it in $\{16, 32, 64\}$ for the graph convolutional layers of all models except for DGCN. 

Some further hyperparameter values are:
\begin{itemize}
    \item MagNet's coefficient $q$ is chosen in $\{0.01, 0.05, 0.1, 0.15, 0.2, 0.25\}$.
    \item The coefficient $\alpha$ used in the PageRank-based models APPNP and DiGraph takes values in $\{0.05, 0.1, 0.15, 0.2\}$.
    \item For APPNP, we set $K = 10$ for the node-classification task (following~\cite{klicpera2018predict}) and take $K$ in $\{1, 5, 10\}$ for the three-class edge prediction task.
    \item For GAT, we select the number of heads in $\{2, 4, 8\}$.
    \item For DGCN,
    %which requires the generation of three order-proximity matrices (1st order proximity, 2nd order in-degree proximity, and 2nd order out-degree proximity). For this network,
    the number of filters per channel takes values in $\{5, 15, 30\}$.
    \item In GIN, the parameter $\epsilon$ is set to 0.
    \item In SSSNET, the parameters $\gamma_s$ and $\gamma_t$ are set to 50 and 0.1, respectively.
    %For the characteristics of the loss function present in the SSSNET model, we set
    10\% of the nodes per class are taken as seed nodes.
    \item In ChebNet, GCN, and QGNN, we employ the symmetrized adjacency matrix defined as $A_s = \frac{A + A^\top}{2}$.
    \item For DiGCL, we select {\em Pacing function} in [{\em linear}, {\em exponential}, {\em logarithmic}, {\em fixed}] using two settings: \textit{i)} $\tau = 0.4$, {\em drop feature rate 1} = 0.3 and {\em drop feature rate 2} = 0.4, and \textit{ii)} $\tau = 0.9$, {\em drop feature rate 1} = 0.2 and {\em drop feature rate 2} = 0.1.
    \item MSGNN's coefficient $q$ is chosen in $\{0.2q_0, 0.4q_0, 0.6q_0, 0.8q_0, q_0 \}$, where $q_0 := 1/[2 \; \text{max}_{u,v}(A_{uv} - A_{vu})]$.
    
\end{itemize}
%\iln{Aggiungere i setting degli altri modleli}

\paragraph{Node Classification Task.} For the {\tt Telegram} dataset, we retain the dataset's original features. When experimenting with the synthetic datasets, the feature vectors are generated using the in-degree and out-degree procedure described before.

\paragraph{Three-Class Edge Prediction Task.} The feature matrix $X \in \mathbb{R}^{n \times 2}$ is defined in such a way that $X_{u1}$ corresponds to the in-degree of node $u$ and $X_{u2}$ to node $u$'s out-degree, for all $u \in N$. This allows the models to learn structural information directly from the adjacency matrices.

\paragraph{Four/Five-Class Edge Prediction Task.} The feature matrix $X \in \mathbb{R}^{n \times 2}$ is defined as in three class edge prediction task. However, for these tasks, the in-degree and out-degree are computed using the absolute values of their corresponding edge weights.

\section{More Experiments}

For completeness, we report the results obtained for the three-class edge prediction task on the {\tt Di500} graphs in Table~\ref{tab:no_performance}. We omitted these results from Subsection \textit{Three-Class Edge Prediction Task (3CEP)} since, as the table shows, on such graphs all the methods achieve a performance close to 33\% (i.e., close to what a purely random model would achieve)---in essence, the methods' predictive capabilities do not deviate significantly from what one would anticipate via a completely random prediction and, as such, these results are not particularly interesting. %These results are due to the small value of $\alpha_{uv}$ considered for these graphs, due to which the graphs turn out to be very sparse, making the task almost impossible to solve.

\begin{table}[htb!]
\centering
\caption{Accuracy (\%) on datasets of the three class edge prediction task}
\label{tab:no_performance}
\begin{tabular}{clll}
\hline
 & \multicolumn{3}{c}{Three Class Edge prediction} \\ \cline{2-4} 
 & \multicolumn{3}{c}{{\tt Di500}} \\ \cline{2-4} 
 & $\delta$ = 0.2 & $\delta$ = 0.5 & $\delta$ = 0.7 \\ \hline
ChebNet & 34.21$\pm$0.01 & 35.14$\pm$0.01 & 35.48$\pm$0.03 \\
GCN & 34.22$\pm$0.01 & {35.15$\pm$0.01} & {35.49${\pm}$0.02} \\
QGNN & 34.22$\pm$0.01 & 35.15$\pm$0.01 & 35.49$\pm$0.02 \\ \hline
APPNP & 34.21$\pm$0.01 & 35.13$\pm$0.01 & 35.48$\pm$0.02 \\
SAGE & 34.18$\pm$0.02 & {35.12$\pm$0.01} & {35.46$\pm$0.02} \\
GIN & 34.08$\pm$0.02 & 35.13$\pm$0.01 & 35.48$\pm$0.01 \\
GAT & 34.21$\pm$0.02 & 35.14$\pm$0.01 & 35.48$\pm$0.02 \\ \hline
DGCN & 34.21$\pm$0.02 & 35.14$\pm$0.02 & 35.48$\pm$0.02 \\
DiGraph & 34.62$\pm$0.16 & 34.76$\pm$0.19 & 35.16$\pm$0.15 \\
DiGCL & 32.89$\pm$0.01 & 32.43$\pm$0.01 & 32.25$\pm$0.01 \\
MagNet & 35.06$\pm$0.14 & 35.02$\pm$0.08 & 35.44$\pm$0.03 \\
\signum{} & {34.93$\pm$0.14} & 34.78$\pm$0.01 & 35.16$\pm$0.12 \\ \hline
\quater{} & {35.07$\pm$0.19} & {34.74${\pm}$0.06} & 35.25$\pm$0.09
\end{tabular}
\end{table}

\end{document}